\documentclass{article}

    \PassOptionsToPackage{numbers, compress}{natbib}

\usepackage[main,final]{neurips_2026}
\usepackage{amsmath}
\usepackage{amssymb}
\usepackage{mathtools}
\usepackage{amsthm}

\usepackage{microtype}
\usepackage{graphicx}
\usepackage{subcaption}
\usepackage{booktabs} 
\usepackage{multirow}
\usepackage{colortbl}

\usepackage{algorithm}
\usepackage{algorithmic}

\usepackage[utf8]{inputenc} 
\usepackage[T1]{fontenc}    
\usepackage{hyperref}       
\usepackage{url}            
\usepackage{amsfonts}       
\usepackage{nicefrac}       
\usepackage{xcolor}         

\theoremstyle{plain}
\newtheorem{theorem}{Theorem}[section]
\newtheorem{proposition}[theorem]{Proposition}
\newtheorem{lemma}[theorem]{Lemma}
\newtheorem{corollary}[theorem]{Corollary}
\theoremstyle{definition}

\theoremstyle{remark}
\newtheorem{remark}[theorem]{Remark}

\title{On the Sparsity-Storage-Accuracy Tradeoff in Parsimoniously Activated Dictionary Learning}

\author{
Zihui Zhao$^{1}$ \quad
Yuanbo Tang$^{1}$ \quad
Yang Li$^{2}$\thanks{Thanks to National Natural Science Foundation of China (62371270).}
\\
Tsinghua University, Institute of Data and Information\\
Shenzhen Key Laboratory of Ubiquitous Data Enabling$^{1}$,\\
School of Artificial Intelligence\\
Chinese University of Hong Kong (Shenzhen)  $^{2}$
}
\begin{document}

\maketitle

\begin{abstract}

Dictionary learning has long been studied from both optimization and probabilistic perspectives. While formulations with element-wise sparsity regularization (e.g., L1-based sparse coding) admit well-established probabilistic interpretations, many structured variants that impose global constraints lack a clear and tractable generative view. In this paper, we revisit a class of practically effective yet theoretically under-explored dictionary learning methods that impose a simple global regularization on the number of activated dictionary atoms, which we term parsimoniously activated dictionary learning (PADL). We show that PADL admits an equivalent formulation as maximum a posteriori estimation under a structured generative model, with auxiliary latent variables that govern global activation patterns. This formulation allows us to derive generalization guarantees that are difficult to obtain under the original formulation. More importantly, it yields an analytical characterization of the tradeoff between sparsity, storage cost, and reconstruction accuracy, enabling data-driven estimation of optimal hyperparameters. Based on this connection, we develop an efficient and interpretable PADL algorithm that eliminates manual hyperparameter tuning, achieving improved reconstruction performance under comparable sparsity levels on visual benchmarks. We further demonstrate its practical utility in accelerating inference for vision-language models.

\end{abstract}

\section{Introduction}

Dictionary learning has long been a fundamental tool in signal processing and machine learning, aiming to find sparse representations of data through an optimization framework~\cite{olshausen1997sparse}. From the optimization perspective, classical sparse coding-based methods, such as K-SVD, formulate dictionary and code estimation as regularized reconstruction problems, yielding scalable algorithms with direct control over reconstruction error and coefficient sparsity~\cite{olshausen1997sparse,aharon2006k,mairal2009online}. From the probabilistic perspective, dictionary learning can be described by generative models over observations and latent coefficients, providing generalization analysis and uncertainty estimates~\cite{paisley2009nonparametric,zhou2009non,dang2017indian,campbell2019sparse,park2025save}. These two views are closely connected for element-wise sparsity constraints. For example, maximum a posteriori (MAP) estimation with Laplace coefficient priors yields the classic $\ell_1$ sparse-coding formulation~\cite{olshausen1997sparse,elad2010sparse}, which makes classical sparse dictionary learning both computationally practical and statistically interpretable.

However, this optimization--probabilistic bridge is much less developed for structured dictionary-learning objectives that impose global constraints on atom usage. These global constraints are particularly important~\cite{dumitrescu2019adaptive}, as they naturally introduce a form of model selection, leading to improved robustness and statistical efficiency. Group-sparse methods encourage block-level selection. For example, simultaneous sparse approximation~\cite{tropp2006algorithms} enforces that multiple sparse signals share the same support; activation-constrained formulations~\cite{krause2010submodular,tang2023explainable} impose constraints on the number of activated atoms; and recent dictionary-learning variants are increasingly used for scalable representation learning~\cite{tropp2006algorithms,krause2010submodular,tang2023explainable}. Such global constraints inherently call for hierarchical modeling. While hierarchical probabilistic approaches to dictionary learning, such as Beta--Bernoulli~\cite{paisley2009nonparametric,zhou2009non} and Indian Buffet Process models~\cite{dang2017indian}, have been studied, they are generally not directly connected to the optimization formulations used in practice.  

Motivated by this gap, we revisit a class of practically effective yet theoretically under-explored dictionary learning methods that impose a simple global regularization on the number of activated dictionary atoms~\cite{krause2010submodular,tang2023explainable}, which we term parsimoniously activated dictionary learning (PADL). Its objective can be written in the unified form as:
\begin{equation}
\label{eq:objective}
\min_{D,R} \|X - D R\|_F^2
+ \lambda_1 \|R\|_1
+ \lambda_2 \sum_{k=1}^K \|r_k\|_\infty ,
\end{equation}
where $X$ is the data matrix, $D$ is the dictionary, $R$ is the coefficient matrix, and $r_k$ is the $k$-th coefficient row. The $\|\cdot\|_F^2$ term controls the reconstruction error and the $\|\cdot\|_1$ term controls sample-level representation sparsity. The row-wise $\ell_\infty$ term constrains global atom activation by penalizing whether a dictionary atom is used by any sample. This formulation is simple, easy to optimize with alternating optimization, and highly effective in fields including trajectory generation and DNA compression~\cite{tangdata,tang2023explainable,adjeroh2002dna}. However, it requires hyperparameter grid search to select optimal $\lambda_1$ and $\lambda_2$, which hinders practical stability and efficiency. Our goal is to provide a generative probabilistic view of this formulation that enables principled optimal hyperparameter computation by characterizing the sparsity--storage--accuracy tradeoff.

We derive this PADL formulation from a structured auxiliary-latent $z$ formulation through MAP estimation with $P(X,R,z;D)$. The coefficient prior $p(R)$ induces per-sample sparsity, while the latent variable $P(z \mid r)$ models whether an atom is activated by the dataset.

This dual formulation retains the computational simplicity of the optimization view while enabling several analytical results from its generative view. In particular, we derive a sufficient condition under which non-essential atoms are suppressed and establish a high-probability reconstruction-error bound. These results make explicit how the activation threshold and regularization weights govern the tradeoff among sparsity, storage, and reconstruction accuracy.

Based on this formulation, we develop an efficient PADL algorithm with analytical data-driven hyperparameter estimation, reducing reliance on exhaustive hyperparameter search. With a smaller activated dictionary size, our method achieves better reconstruction quality and significantly reduces the time required for parameter search. We evaluate the framework on both classical image reconstruction benchmarks and Vision--Language Model (VLM) inference acceleration via token compression. On image reconstruction tasks, our method achieves superior reconstruction performance and empirically aligns with the theoretical predictions on optimal hyperparameter selection and dictionary activation thresholds. On VLM inference acceleration tasks, our method consistently attains higher token compression ratios while preserving downstream task performance.

Our contributions can be summarized as follows:
\begin{itemize}
    \item We revisit PADL, a storage-aware dictionary-learning framework that regularizes cross-sample atom activation and directly targets the sparsity--storage--accuracy tradeoff.
    \item We provide a structured MAP-induced interpretation of PADL, linking the row-wise activation penalty to an atom-level auxiliary latent activation model.
    \item We derive suppression conditions, expected activated dictionary size, reconstruction-error bounds, and a data-driven activation-threshold rule.
    \item We validate the framework on controlled image reconstruction benchmarks and report image and video VLM compression results as application-level validation.
\end{itemize}

\section{Method}

\subsection{Problem Formulation}
Given a data set $X=\{X_i\}_{i=1}^N$, where each $X_i \in \mathbb{R}^d$ corresponds to the feature vector of a sample, we introduce a dictionary $D=[d_1,\ldots,d_K]\in\mathbb{R}^{d\times K}$. Each sample is generated as
\[
X_i = D R_i + \epsilon_i,
\]
where $R_i \in [0,1]^{K \times 1}$ and $\epsilon_i$ is isotropic Gaussian noise, yielding the likelihood
\[
p(X_i \mid R_i; D)\propto
\exp\!\left(-\tfrac{1}{2}\|X_i - D R_i\|_2^2\right).
\]
Specifically, we assume each coefficient follows an independent Beta distribution,
\begin{equation}
R_{k i} \sim \mathrm{Beta}(1,\beta), \qquad \beta > 1,
\end{equation}
where $1\le k\le K$ with density
\begin{equation}
p(R_{k i}) = \beta (1 - R_{k i})^{\beta-1}, \qquad R_{k i} \in [0,1].
\end{equation}
This prior concentrates mass near zero and naturally favors sparse activations. To model the storage cost, we introduce a latent activation variable $z_k \in \{0,1\}$ for each dictionary atom $d_k$, indicating whether the atom is effectively used across the samples. Let $r_k \in \mathbb{R}^N$ denote the $k$-th row of $R$, i.e., the coefficients of atom $k$ across all samples.
Conditioned on $r_k$, we define
\begin{equation}
z_k \mid r_k \sim \operatorname{Bernoulli}
\left(
\frac{1}{1 + \exp\!\big( -\gamma (\max_i R_{k i} - \delta) \big)}
\right),
\label{eq:bernoulli_def}
\end{equation}
where $\delta \in [0,1]$ denotes an activation threshold and $\gamma > 0$ controls the sharpness of the decision. The quantity $\max_i R_{k i}$ summarizes whether atom $k$ contributes significantly to at least one sample. An atom $d_k$ is considered active only if $r_k$ attains a sufficiently large coefficient for some sample (above the threshold). The threshold $\delta$ therefore acts as a direct control on activation, while the Beta prior governs per-sample sparsity. We can therefore define a joint distribution
\begin{equation}
\label{eq:joint_distribution}
\begin{aligned}
P(X,R,z;D)
=&\;\prod_{i=1}^N P(X_i\mid R_i;D)
\cdot \prod_{k=1}^K\prod_{i=1}^N p(R_{ki})
\cdot \prod_{k=1}^K P(z_k\mid r_k).
\end{aligned}
\end{equation}

\subsection{MAP Objective and Algorithm}

Under the probabilistic model introduced above, we apply maximum a posteriori (MAP) to learn model parameters given data set $X$. Taking the negative log-posterior yields an energy of the form
\begin{equation}
\begin{aligned}
-\log P(R,z \mid X; D) 
=
\sum_{i=1}^N \frac{1}{2}\|X_i - D R_i\|_2^2
-
\sum_{k,i} \log p(R_{k i})
- 
\sum_k \log P(z_k=0\mid r_k)
+ C,
\end{aligned}
\end{equation}
where $C$ is a constant independent of $(R,z)$.
\begin{proposition}[MAP-Induced Objective]
\label{prop:map_objective}
The MAP estimation of joint distribution $P(X,R,z;D) $ as described in Eq.~\ref{eq:joint_distribution} can be formulated as the following optimization problem:
\begin{equation}
\label{eq:novel-obj-main}
\min_{D\in R^{d\times K},R\in[0,1]^{K\times N}}
\;\;
\|X - D R\|_F^2
+
\lambda_1 \|R\|_1
+
\lambda_2
\sum_{k=1}^K
\phi_\gamma\!\big(\max_i R_{k i} - \delta\big),
\end{equation}
where $\phi_\gamma(u)=\log(1+e^{\gamma u})$ , with $\lambda_1=\beta-1$ and $\lambda_2=1$. For large $\gamma$, $\phi_\gamma(u)$ approaches a hard threshold barrier on $\max_iR_{ki}\le\delta$. Then Eq.~\ref{eq:novel-obj-main} is equivalent to:
\begin{equation}
\label{eq:novel-obj-hard}
\min_{D\in R^{d\times K},R\in[0,1]^{K\times N}}
\;\;
\|X - D R\|_F^2
+
\lambda_1 \|R\|_1
+
\lambda_2\sum_{k=1}^K \|r_k\|_\infty ,
\end{equation}
where $\|r_k\|_\infty=\max_i |R_{k i}|$. 
\end{proposition}
\begin{proof}
    See Appendix~\ref{ap:map_obj}.
\end{proof}
Intuitively, the Beta prior induces an $\ell_1$ penalty that promotes representation sparsity, while the activation prior gives rise to a soft barrier on $\max_i R_{k i}$. The row-wise $\ell_\infty$ term is the convex surrogate used to penalize the same peak-activation cost in the practical PADL objective. 

Unlike classical sparse coding objectives, which balance only reconstruction and coefficient sparsity, the PADL formulation explicitly models dictionary size as an optimization objective. The threshold $\delta$ and the associated $\ell_\infty$ term act as direct control parameters for dictionary size, enabling explicit regulation of the sparsity-storage-accuracy tradeoff.

This formulation connects the hyperparameters $\lambda_1$ and $\lambda_2$  to the prior parameter $\beta$. In the next subsection, we illustrate how the error bound is affected by the hyperparameters and the activation threshold $\delta$.

\subsection{Tradeoff Characterization}

We now characterize how the PADL objective governs the interaction among reconstruction accuracy, representation sparsity, and dictionary usage. 
\begin{proposition}[Suppression of Non-essential Atoms]
\label{prop:suppress}
Assume that a subset of atoms $S$ can reconstruct the data with bounded error $\epsilon$. For any additional atom $k\notin S$ with nonzero coefficient row $r_k$, let $e=X-D_SR_S$ and suppose $|e_i^\top d_k|\le \eta$ for all $i$. If
\begin{equation}
\label{eq:suppress_cond}
\lambda_1
+
\lambda_2\frac{\|r_k\|_\infty}{\|r_k\|_1}
\;>\;
2\eta,
\end{equation}
then setting the coefficient vector $r_k$ to zero strictly decreases the objective. In particular, for atoms whose peak activation exceeds the threshold, $\|r_k\|_\infty\ge\delta$, it is sufficient that
\begin{equation}
\label{eq:suppress_cond_delta}
\lambda_1
+
\lambda_2\frac{\delta}{\|r_k\|_1}
\;>\;
2\eta.
\end{equation}

\end{proposition}
\begin{proof}
    See Appendix~\ref{ap:sufficient_cond}
\end{proof}
\begin{figure}
    \centering
    \includegraphics[width=0.7\linewidth]{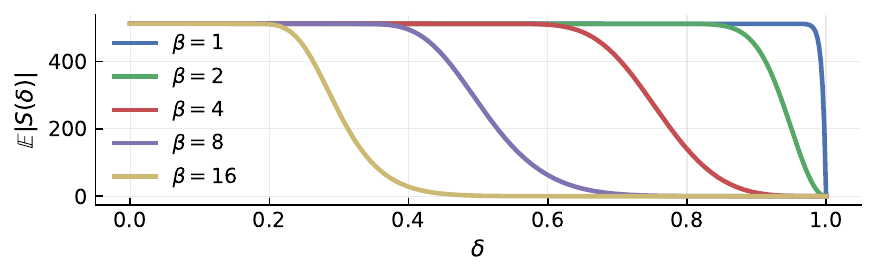}
    \caption{The estimated dictionary activation size as a function of $\delta$ under different values of $\beta$, estimated on CIFAR-100.}
    \label{fig:es_vs_delta}
\end{figure}

This result formalizes the role of the activation penalty: atoms that only exhibit occasional or weak activations are eliminated, even if their coefficients satisfy the $\ell_1$ sparsity constraint. Such atoms correspond to spurious patterns that are likely caused by noise. The model therefore favors visual primitives that are consistently useful, yielding a compact and reusable dictionary. 
\begin{proposition}[Expected Activated Dictionary Size]
\label{prop:ES_delta}
Assume that the activation coefficients satisfy
$R_{ki}\overset{i.i.d.}{\sim}\mathrm{Beta}(1,\beta)$ for
$k=1,\ldots,K$ and $i=1,\ldots,N$.
Define the active atom set induced by threshold $\delta$ as
\[
S(\delta)
\;:=\;
\{\,k \;:\; \max_{i} R_{ki} > \delta \,\}.
\]
Let $F(r)=1-(1-r)^{\beta}$ denote the cumulative distribution function
of $\mathrm{Beta}(1,\beta)$.
Then the expected size of the active set is
\begin{equation}
\label{eq:ES_delta}
\begin{aligned}
    \mathbb{E}\,|S(\delta)|
\;=\;
K\bigl(1 - F(\delta)^N\bigr)
\;=&\;
K\left[1 - \bigl(1-(1-\delta)^{\beta}\bigr)^N\right].
\end{aligned}
\end{equation}
In particular, $\mathbb{E}|S(\delta)|$ is a strictly decreasing and
smooth function of $\delta$, with
\[
\mathbb{E}|S(0)| = K,
\qquad
\mathbb{E}|S(1)| = 0.
\]
\end{proposition}

\begin{proof}
    See Appendix~\ref{ap:exp_s}
\end{proof}
The activated dictionary size $\mathbb{E}|S(\delta)|$ is strictly decreasing in $\delta$, as shown in Figure~\ref{fig:es_vs_delta}, which indicates that a larger activation threshold decreases the dictionary size by activating fewer atoms.
\begin{theorem}[High-Probability Reconstruction Bound]
\label{prop:err_bound}
Under the active-span restricted-eigenvalue condition stated in Appendix~\ref{ap:err_bound}, the average reconstruction error satisfies, with probability at least $1-\tau$,
\begin{equation}
\frac{\|e\|_F^2}{N}
\;\le\;
C(|S|)
\left(
\lambda_1
+
\frac{\lambda_2\delta}{U(N,\beta,\tau)}
\right)^2,
\end{equation}
where $C(|S|)$ grows linearly with $|S|$ and
\[
U(N,\beta,\tau)
:=
\frac{N}{1+\beta}
+
\sqrt{\frac{N}{2}\log\frac{1}{\tau}}
\]
grows linearly with $N$.
\end{theorem}
\begin{proof}
    See Appendix~\ref{ap:err_bound}
\end{proof}
The average error bound depends jointly on the size of the activated dictionary and the activation threshold, both of which are inherently controlled by the threshold.
\begin{corollary}[Optimal Activation Threshold Estimation]
For fixed data size, data distribution, and confidence level, the optimal threshold $\delta^*$ satisfies:
\begin{equation}
\label{eq:delta_stationary}
\frac{N\,F(\delta)^{N-1}\,f(\delta)}{1-F(\delta)^N}
=
\frac{2\lambda_2}{U(N,\beta,\tau)\left(\lambda_1+
\frac{\lambda_2\delta}{U(N,\beta,\tau)}\right)}.
\end{equation}
Here,
\[
F(\delta) := \Pr(R_{k i}\le \delta) = 1-(1-\delta)^\beta,
f(\delta) := F'(\delta)=\beta(1-\delta)^{\beta-1}.
\]
\end{corollary}
\begin{proof}
    See Appendix~\ref{ap:opt_delta}
\end{proof}
As the number of samples $N$ increases, $U(N,\beta,\tau)$ grows linearly, so the threshold-dependent contribution in the reconstruction bound decays as $\mathcal{O}(1/N)$. This indicates that the threshold term does not inflate the average error bound, while the active dictionary size remains governed by the choice of $\delta$.

Together, these results provide a theoretical characterization of the activated dictionary size and reconstruction-error bound. They reduce the complex sparsity-storage-accuracy tradeoff to a single controlled variable $\delta$ with a computable optimal value, thereby reducing the need for tedious hyperparameter grid search with heuristic settings.

\subsection{Dictionary Learning Algorithm}

Using the above theoretical results, we proposed a modified PADL algorithm with analytical hyperparameter estimation from data, as shown in Algorithm~\ref{alg:dictionary_learning}. This algorithm adopts alternating optimization~\cite{mairal2009online} with projected gradient descent to jointly learn the dictionary and the sparse representation while ensuring convergence, a standard technique in previous literature~\cite{tseng2001convergence,attouch2013convergence,bolte2014proximal}. To estimate the hyperparameters, we first perform a warm-up phase with tentative initialization, during which we obtain the Beta parameter and compute the corresponding $\lambda$ and $\delta^\ast$, which are then used in the subsequent dictionary learning process. During testing, the dictionary remains fixed, and the hyperparameters and activation threshold are set to their estimated values; only the sparse representation $R$ is updated. Note that, while the hyperparameters could be updated iteratively during alternating optimization to obtain a more precise estimate of the prior distribution, we empirically find that one-shot warm-up estimation is sufficient in our experiments, as shown in Appendix~\ref{app:warmup_sensitivity}.

\begin{algorithm}[t]
\caption{Parsimoniously Activated Dictionary Learning}
\label{alg:dictionary_learning}
\begin{algorithmic}[1]
\REQUIRE 
Data matrix $\mathbf{X}$; 
dictionary size $K$; 
outer iterations $T$; 
inner steps $k_r, k_d$;
warm-up steps $t_s$
\ENSURE 
Learned dictionary $\mathbf{D}\in\mathbb{R}^{d\times K}$

\STATE Use preprocessed data:
$\mathbf{X}_{\text{norm}} \gets \mathbf{X}$

\STATE Initialize $\mathbf{D}\sim\mathcal{U}[0,1]$, $\mathbf{R}\sim \mathrm{Beta}(1,1)$
\STATE Initialize $(\lambda_1,\lambda_2,\delta) \gets (0,1,0)$ as warm-up defaults

\FOR{$t=1$ to $T$}
        
    \IF{$t = t_s$}
        \STATE Fit empirical $\mathbf{R}$ with $\mathrm{Beta}(1,\beta)$
        \STATE Compute $\lambda_1\approx\beta-1$, $\lambda_2\gets 1$, and $\delta^\ast$ using Eq.~\ref{eq:delta_stationary}
        \STATE Re-initialize $\mathbf{D}\sim\mathcal{U}[0,1]$, 
        $\mathbf{R}\sim \mathrm{Beta}(1,\beta)$
        
    \ENDIF

    \FOR{$\ell=1$ to $k_r$}
        \STATE Update $\mathbf{R}$ by solving $\min_{\mathbf{R}}\;
        \|\mathbf{X}_{\text{norm}}-\mathbf{D}\mathbf{R}\|_F^2
        + \lambda_1\|\mathbf{R}\|_1
        + \lambda_2 \sum_{k=1}^K \|r_k\|_\infty$.
        \STATE Project coefficients: $\mathbf{R}\gets \Pi_{[0,1]}(\mathbf{R})$.
    \ENDFOR

    \FOR{$\ell=1$ to $k_d$}
        \STATE Update $\mathbf{D}$ by solving $\min_{\mathbf{D}}\;
        \|\mathbf{X}_{\text{norm}}-\mathbf{D}\mathbf{R}\|_F^2
        \ \text{s.t.}\ \|d_k\|_2=1,\ \forall k$.
    \ENDFOR

\ENDFOR

\STATE \textbf{return} $\mathbf{D}$
\end{algorithmic}
\end{algorithm}

\subsection{VLM Inference Compression}

PADL can also be used as a structured visual compression module for VLM
inference. Given visual tokens $\mathbf{Z}_V$ produced by a vision encoder, we
learn a visual dictionary $\mathbf{D}_V$ and sparse codes $\mathbf{R}_V$ such
that $\mathbf{Z}_V \approx (\mathbf{D}_V\mathbf{R}_V)^T$. The row-wise activation
regularizer selects a compact active sub-dictionary $\mathbf{D}_S$, while the
sparse codes can be quantized for transmission or storage. As shown in
Figure~\ref{fig:imp_vlm}, the compressed representation can be used in two
ways: reconstructing visual tokens before feeding them to the LLM, or passing
the sparse representation through a lightweight adapter for direct inference.
The detailed VLM formulation, compression protocol, and adapter setting are
provided in Appendix~\ref{app:vlm_experiments}.

\begin{figure}[t]
    \centering
    \includegraphics[width=0.5\linewidth]{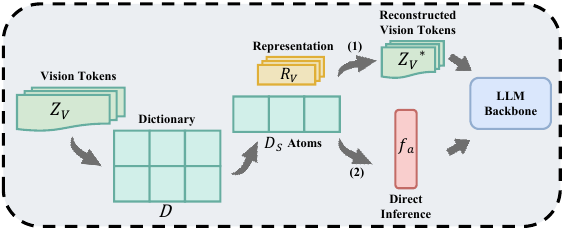}
    \caption{Two implementations on VLMs are considered: reconstructing visual tokens and performing direct inference via a fine-tuned adapter.}
    \label{fig:imp_vlm}
\end{figure}

\section{Experiment}
Our experimental evaluation focuses on validating the theoretical contributions of our framework, particularly the Bayesian interpretation and the parameter selection guidelines. We conduct experiments on CIFAR-100~\cite{krizhevsky2009learning} and SVHN~\cite{netzer2011reading} to demonstrate the practical implications of our theoretical framework. Each dataset contains color images of size 32 $\times$ 32. All images are converted to grayscale, and non-overlapping 8 $\times$ 8 patches are extracted for training and evaluation. The dictionary size is fixed at 128 atoms across all experiments. We use 20,000 random patches for training and reserve 2,000 patches for testing and performance evaluation. Unless otherwise specified, the number of activated atoms denotes the number of dictionary atoms satisfying $\max_i |R_{ki}|>\delta^\ast$ on the evaluation set, using the same thresholding rule across methods.

For VLM applications, we conduct experiments with LLaVA-1.5~\cite{liu2024improved} on VQA tasks for image inputs and Video-LLaVA~\cite{lin2024video} for video inputs. The dictionary size is fixed at 1024 atoms. For image inputs, our evaluation includes VQAv2~\cite{goyal2017making}, ScienceQA~\cite{lu2022learn}, TextVQA~\cite{singh2019towards}, POPE hallucination benchmark~\cite{li2023evaluating}, MME~\cite{fu2023mme} and MMBench~\cite{liu2024mmbench}. We apply $4\times$ dimensional compression together with $2\times$ quantization. The adapter layer is fine-tuned on the LLaVA-1.5 instruction-tuning dataset~\cite{liu2024improved}. For video inputs, our evaluation is conducted on MSVD-QA~\cite{chen2023x}, MSRVTT-QA~\cite{xu2016msr}, TGIF-QA~\cite{jang2017tgif} and ActivityNet-QA~\cite{yu2019activitynet}. We apply $8\times$ dimensional compression together with $2\times$ quantization. The adapter layer is fine-tuned on a dataset provided by Valley ~\cite{luo2023valley}. We compare our method with two SOTA token reduction methods: LLaVA-PruMerge~\cite{shang2024llava} and freePruner~\cite{xu2024freepruner}. The variant without fine-tuning is denoted as \textbf{PADL}, while the variant with fine-tuning is denoted as \textbf{PADL+}. For image benchmarks, we follow the standard evaluation metrics of each benchmark. For video QA, we report accuracy (\%) and GPT-based answer-quality scores assessed using GPT-3.5-Turbo, following prior work.

Experiments were run on Ubuntu 22.04.5 with two Intel Xeon Gold 6348 CPUs, and a single NVIDIA A800 80GB PCIe GPU. The CUDA version was 12.2.

\subsection{Hyperparameter Estimation Validation}

\begin{table}[h]
\centering
\caption{Validation of theoretical parameter estimates.}
\label{tab:parameter_validation}
\small
\setlength{\tabcolsep}{3pt}
\begin{tabular}{lccc}
\hline
\textbf{Parameter} & \textbf{Theoretical} & \textbf{Empirical} & \textbf{Relative Error $\%$}\\
\hline
$\lambda_1$ & 1.31 & 1.33 & 1.5\\
$\lambda_2$ & 1.00 & 1.01 & 1.0\\

\hline
\end{tabular}
\end{table}

To validate the precision of our theoretically predicted hyperparameter settings, we compared the theoretically recommended parameters with empirically optimal values found through extensive grid search. Table~\ref{tab:parameter_validation} reports the results and corresponding relative errors. 

\begin{figure}[t]
    \centering
    \begin{minipage}[t]{0.40\linewidth}
        \centering
        \includegraphics[width=\linewidth]{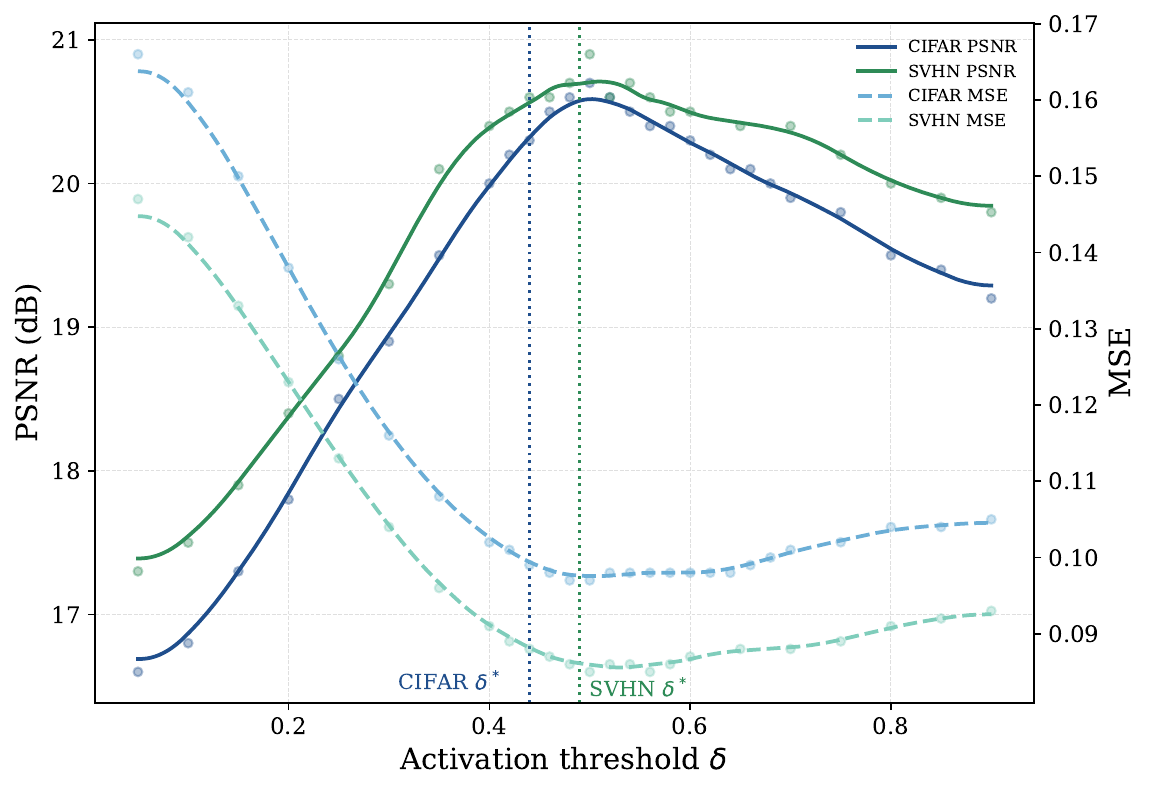}
        \caption{Reconstruction error under varying activation thresholds $\delta$ on CIFAR-100 and SVHN; dashed lines indicate the predicted $\delta^*$.}
        \label{fig:error-delta}
    \end{minipage}
    \hfill
    \begin{minipage}[t]{0.5\linewidth}
        \centering
        \includegraphics[width=\linewidth]{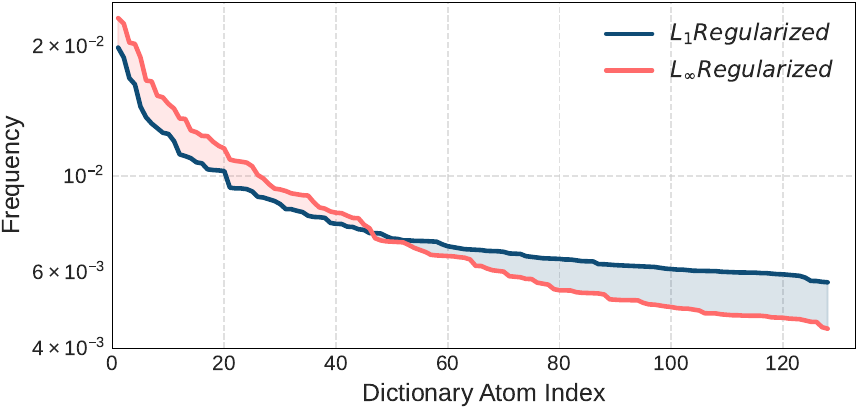}
        \caption{Atom utilization frequency on CIFAR-100 under $L_\infty$- and $L_1$-regularized dictionaries.}
        \label{fig:atoms_sorted}
    \end{minipage}
\end{figure}

To validate the error trend under different activation thresholds $\delta$, we conduct experiments with varying $\delta$ values on both datasets. As shown in Figure~\ref{fig:error-delta}, the theoretically predicted optimal $\delta^*$ is indicated by the vertical dashed lines. These results show that our theoretical predictions closely align with the empirical optima, supporting the derived threshold-selection rule and the associated theoretical characterization.

\subsection{Parsimonious Activation Validation}

Figure~\ref{fig:atoms_sorted} presents a comparison of dictionary atom usage frequencies between the $L_\infty$- and $L_1$-regularized dictionaries, evaluated under identical parameter configurations and averaged over the CIFAR-100 test set. The usage in the PADL model is concentrated on a smaller subset of the original dictionary, suggesting that the regularization on $R$ encourages a more compact and efficient representation. 

\subsection{Comparative Reconstruction Performance}

While our primary focus is theoretical validation, we include a concise comparison of reconstruction performance to demonstrate practical utility. Table~\ref{tab:reconstruction_comparison} shows that the PADL method achieves competitive reconstruction quality while providing theoretical guarantees.

\begin{table}[t]
\centering
\caption{Reconstruction performance comparison on CIFAR-100 and SVHN. For recent baselines and PADL, results are reported as mean $\pm$ standard deviation over 10 repeated runs.}
\label{tab:reconstruction_comparison}
\small
\setlength{\tabcolsep}{3pt}
\resizebox{\linewidth}{!}{%
\begin{tabular}{lcccccc}
\hline
\multirow{2}{*}{\textbf{Method}} &
\multicolumn{3}{c}{\textbf{CIFAR-100}} &
\multicolumn{3}{c}{\textbf{SVHN}} \\
\cline{2-7}
& \textbf{RMSE} & \textbf{PSNR} & \textbf{SSIM}
& \textbf{RMSE} & \textbf{PSNR} & \textbf{SSIM} \\
\hline
KSVD~\cite{aharon2006k} & 0.152 & 17.87 & 0.472 & 0.147 & 16.34 & 0.449 \\
DDL~\cite{tariyal2016deep} & 0.132 & 18.07 & 0.514 & 0.124 & 18.01 & 0.527 \\
GDDL~\cite{tariyal2016greedy} & 0.123 & 18.32 & 0.519 & 0.118 & 19.33 & 0.539 \\
BPFA~\cite{paisley2009nonparametric} & 0.127 & 18.14 & 0.520 & 0.121 & 18.72 & 0.529 \\
IBP~\cite{dang2017indian} & 0.119 & 18.58 & 0.535 & 0.119 & 19.57 & 0.545 \\
CRsAE~\cite{tolooshams2020deep} & $0.106 \pm 0.007$ & $20.13 \pm 0.56$ & $0.558 \pm 0.035$ & $0.092 \pm 0.006$ & $20.30 \pm 0.60$ & $0.565 \pm 0.025$ \\
JumpReLU SAE~\cite{rajamanoharan2024jumping} & $0.104 \pm 0.016$ & $19.58 \pm 0.97$ & $0.520 \pm 0.041$ & $0.091 \pm 0.010$ & $20.01 \pm 1.19$ & $0.518 \pm 0.042$ \\
$L_1$-Regularized DL & $0.112 \pm 0.006$ & $18.98 \pm 0.54$ & $0.545 \pm 0.030$ & $0.094 \pm 0.012$ & $20.46 \pm 1.03$ & $0.559 \pm 0.037$ \\
\textbf{PADL} & $\mathbf{0.099 \pm 0.009}$ & $\mathbf{20.56 \pm 0.55}$ & $\mathbf{0.568 \pm 0.030}$ & $\mathbf{0.090 \pm 0.006}$ & $\mathbf{20.88 \pm 1.03}$ & $\mathbf{0.583 \pm 0.037}$ \\
\hline
\end{tabular}
}
\end{table}

\subsection{Computational Cost}

\begin{table}[t]
\centering
\begin{minipage}[t]{0.57\linewidth}
\centering
\captionof{table}{Computational-cost summary on CIFAR-100 using 2,000 training patches.}
\label{tab:main_computational_cost}
\scriptsize
\setlength{\tabcolsep}{2pt}
\renewcommand{\arraystretch}{1.08}
\resizebox{\linewidth}{!}{
\begin{tabular}{lcc}
\toprule
Method & Cost & Outcome \\
\midrule
K-SVD~\cite{aharon2006k} & 661.4s & 20 atoms / 0.152 RMSE \\
CRsAE~\cite{tolooshams2020deep} & 32.0s & 56 atoms / 0.106 RMSE \\
JumpReLU SAE~\cite{rajamanoharan2024jumping} & 3.1s & 61 atoms / 0.104 RMSE \\
PADL & 44.2s & 19 atoms / 0.099 RMSE \\
\bottomrule
\end{tabular}
}
\end{minipage}
\hfill
\begin{minipage}[t]{0.39\linewidth}
\centering
\captionof{table}{Ablation study on CIFAR-100.}
\label{tab:ablation_study}
\scriptsize
\setlength{\tabcolsep}{2pt}
\renewcommand{\arraystretch}{1.08}
\resizebox{\linewidth}{!}{
\begin{tabular}{lccc}
\toprule
Configuration & RMSE & PSNR & SSIM \\
\midrule
PADL & 0.099 & 20.56 & 0.568 \\
w/o $L_\infty$ & 0.112 & 18.98 & 0.545 \\
w/o $L_1$ & 0.108 & 19.17 & 0.549 \\
w/o $L_1$ \& $L_\infty$ & 0.117 & 18.49 & 0.525 \\
\bottomrule
\end{tabular}
}
\end{minipage}
\end{table}

Since PADL introduces a row-wise activation regularizer, we explicitly report training cost in the main text. The cost study uses a smaller 2,000-patch training subset to make wall-clock comparisons across baselines tractable. Table~\ref{tab:main_computational_cost} summarizes training-time cost on the controlled dictionary-learning benchmark. Although PADL is not the fastest training procedure among all baselines, it remains within a practical wall-clock range while producing the smallest activated dictionary and the lowest RMSE in this comparison. Detailed training-time and warm-up-cost studies are provided in Appendices~\ref{app:training_time_atoms} and~\ref{app:warmup_cost}.

\subsection{Ablation Studies}

\begin{figure}
    \centering
    \includegraphics[width=0.6\linewidth]{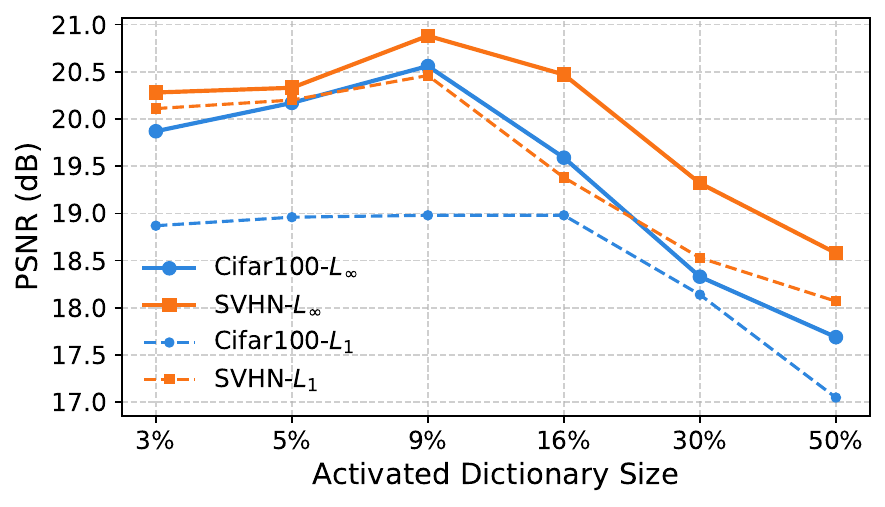}
    \caption{PSNR comparison on CIFAR-100 and SVHN under varying dictionary activation levels.}
    \label{fig:sparsity_sensitivity}
\end{figure}

We conducted ablation studies to evaluate the contribution of each component in our framework. Table~\ref{tab:ablation_study} shows the performance of different configurations. Figure~\ref{fig:sparsity_sensitivity} shows the sensitivity of reconstruction error to sparsity variations. Compared with the $L_1$-regularized approach, PADL achieves better reconstruction performance while using fewer activated dictionary atoms. The results show that both regularization terms contribute significantly to the performance, with the $\ell_\infty$ term providing particularly important theoretical properties.

\subsection{VLM Compression Results}

We further evaluate PADL as a visual compression module for image and video VLMs. Table~\ref{tab:img_qa_app} reports image VQA results. Compared with the original LLaVA-1.5 backbone, PADL achieves substantial visual representation compression while largely preserving downstream performance. With adapter fine-tuning, PADL+ recovers or improves performance on several benchmarks. Table~\ref{tab:video_qa_app} reports video QA results. PADL achieves a $16\times$ representation-size reduction, and PADL+ further improves downstream reasoning performance after adapter fine-tuning. Detailed formulation, compression settings, and threshold validation are
provided in Appendix~\ref{app:vlm_experiments}.

\begin{table*}[t]
\centering
\small
\setlength{\tabcolsep}{6pt}
\renewcommand{\arraystretch}{1.15}
\caption{Performance comparison on image VQA benchmarks. PADL is used as a visual-representation compression module. PADL denotes the version without adapter fine-tuning, while PADL+ denotes the version with adapter fine-tuning.}
\label{tab:img_qa_app}
\begin{tabular}{l l| c c c | c c c}
\toprule
Method & LLM  & VQA$^{v2}$ & SQA$^{I}$ & VQA$^{T}$ & POPE & MME & MMB \\
\midrule
InstructBLIP      & Vicuna-7B  & -    & 60.5 & 50.1 & -    & -      & 36.0 \\
InstructBLIP      & Vicuna-13B & -    & 63.1 & 50.7 & 78.9 & 1212.8 & - \\
Shikra            & Vicuna-13B & 77.4 & -    & -    & -    & -      & 58.8 \\
IDEFICS-9B        & LLaMA-7B   & 50.9 & -    & 25.9 & -    & -      & 48.2 \\
IDEFICS-80B       & LLaMA-65B  & 60.0 & -    & 30.9 & -    & -      & 54.5 \\
Qwen-VL           & Qwen-7B    & \textbf{78.8} & 67.1 & \textbf{63.8} & -    & -      & 38.2 \\

\midrule
LLaVA-1.5         & Vicuna-7B  & 78.5 & 66.8 & 58.2 & 85.9 & \textbf{1510.7} & 64.3 \\
LLaVA-1.5 + PruMerge & Vicuna-7B  & 76.6 & 67.5 & 55.6 & 86.5 & 1414.0 & 62.9 \\
LLaVA-1.5 + FreePruner & Vicuna-7B & 77.6 & 68.6 & 60.0 & 87.7 & 1485.2 & 63.8 \\
\rowcolor{gray!12}
LLaVA-1.5 + \textbf{PADL} ($8\times$) & Vicuna-7B & 76.9 & 65.5 & 57.7 & 84.6 & 1473.1 & 64.1 \\
\rowcolor{gray!12}
LLaVA-1.5 + \textbf{PADL+} ($8\times$) & Vicuna-7B & 78.2 & \textbf{68.8} & 61.0 & \textbf{88.5} & 1485.1 & \textbf{64.7} \\
\bottomrule
\end{tabular}
\end{table*}

\begin{table*}[t]
\centering
\small
\setlength{\tabcolsep}{5pt}
\renewcommand{\arraystretch}{1.2}
\caption{Performance comparison on video reasoning benchmarks. }
\label{tab:video_qa_app}
\resizebox{\textwidth}{!}{
\begin{tabular}{lcccccccccc}
\toprule
\textbf{Methods} & \textbf{LLM size}
& \multicolumn{2}{c}{\textbf{MSVD-QA}}
& \multicolumn{2}{c}{\textbf{MSRVTT-QA}}
& \multicolumn{2}{c}{\textbf{ActivityNet-QA}}
& \multicolumn{2}{c}{\textbf{TGIF-QA}} \\
\cmidrule(lr){3-4}
\cmidrule(lr){5-6}
\cmidrule(lr){7-8}
\cmidrule(lr){9-10}
&  & \textbf{Accuracy} & \textbf{Score}
& \textbf{Accuracy} & \textbf{Score}
& \textbf{Accuracy} & \textbf{Score}
& \textbf{Accuracy} & \textbf{Score} \\
\midrule
VideoChat & 7B & 56.3 & 2.8 & 45.0 & 2.5 & -- & 2.2 & 34.4 & 2.3 \\
LLaMA-Adapter & 7B & 54.9 & 3.1 & 43.8 & 2.7 & 34.2 & 2.7 & - & - \\
Video-ChatGPT & 7B & 64.9 & 3.3 & 49.3 & 2.8 & 35.2 & 2.7 & 51.4 & 3.0 \\
\midrule
Video-LLaVA & 7B & 70.7 & 3.9 & 59.2 & 3.5 & 45.3 & 3.3  & 70.0 & 4.0 \\
Video-LLaVA + PruMerge & 7B & 71.1 & 3.9 & 58.4 & 3.5 & 48.3 & 3.4 & - & - \\
Video-LLaVA + FreePruner & 7B & 71.3 & 3.9 & 59.5 & 3.5 & 48.4 & 3.4 & - & - \\
\rowcolor{gray!12}
Video-LLaVA + \textbf{PADL} ($16\times$) & 7B & 72.3 & 3.9 & 61.1 & 3.5 & 48.1 & 3.4 & 69.9 & 4.1 \\
\rowcolor{gray!12}
Video-LLaVA + \textbf{PADL+} ($16\times$) & 7B & \textbf{75.2} & \textbf{4.1} & \textbf{64.1} & \textbf{3.7} & \textbf{49.7} & \textbf{3.4} & \textbf{75.3} & \textbf{4.3} \\
\bottomrule
\end{tabular}
}
\end{table*}

\section{Conclusion}

In this work, we revisit PADL, a dictionary-learning formulation for balancing sparsity, storage, and accuracy, through a structured probabilistic perspective. By connecting the PADL objective to an auxiliary latent-variable MAP-induced formulation, we obtain interpretable parameter-calibration rules, suppression conditions, activated-dictionary-size estimates, reconstruction-error bounds, and a solvable activation-threshold condition. Compared with dictionary-learning analyses that focus primarily on reconstruction accuracy and coefficient sparsity, this view explicitly accounts for the activated dictionary size, which directly corresponds to storage cost. The resulting calibration procedure is simple to implement and reduces reliance on exhaustive hyperparameter search. Experiments on classical image reconstruction benchmarks validate the predicted sparsity--storage--accuracy behavior, and additional VLM experiments suggest that the same principle can be useful for sparse visual-representation compression.

\bibliographystyle{plainnat}
\bibliography{refs}

\appendix


\section{Related Works}

\subsection{Dictionary Learning and Sparse Representation}

Traditional dictionary learning aims to represent data using a small number of dictionary atoms while preserving reconstruction fidelity. Sparse coding and the seminal work of Olshausen and Field~\cite{olshausen1997sparse} established this formulation as a central model for sparse representation, and subsequent algorithms such as K-SVD~\cite{aharon2006k} and online dictionary learning~\cite{mairal2009online} made dictionary learning scalable to larger datasets. These classical methods primarily control sparsity at the coefficient level, typically through explicit sparsity constraints or $\ell_1$-type regularization. In this paper, we focus on a different but related quantity: the number of dictionary atoms that are effectively activated across a dataset, which is directly tied to storage and deployment cost.

From the probabilistic perspective, dictionary learning is formulated through generative models over observations, latent coefficients, and sometimes latent usage variables. Bayesian sparse coding and nonparametric dictionary-learning models use such variables to infer sparsity patterns and dictionary size from data~\cite{paisley2009nonparametric,zhou2009non,dang2017indian,dang2018towards}. Closely related models such as binary sparse coding~\cite{henniges2010binary} and discrete sparse coding~\cite{exarchakis2017discrete} introduce discrete latent variables to represent sparse activations. These formulations provide a natural way to reason about posterior uncertainty, latent dimensionality, sample complexity, and generalization~\cite{paisley2009nonparametric,zhou2009non,dang2017indian}. In these probabilistic formulations, the latent usage variables are commonly part of the reconstruction mechanism, for example by gating sample-level coefficients or factors~\cite{paisley2009nonparametric,zhou2009non,dang2017indian,henniges2010binary,exarchakis2017discrete}. Our formulation instead introduces an auxiliary atom-level activation variable that does not directly participate in reconstruction, but induces a row-wise activation prior corresponding to cross-sample atom usage.

From the optimization perspective, dictionary learning is commonly posed as a regularized reconstruction problem. This view leads to efficient objectives and algorithms, and it supports theoretical analyses of sparse recovery, stability, identifiability, and convergence~\cite{chen2001atomic,papyan2017working,papyan2018theoretical}. Classical sparse coding, K-SVD, and online dictionary learning fall into this category~\cite{olshausen1997sparse,aharon2006k,mairal2009online}. Their regularizers are usually imposed on individual coefficients or sample-level sparse codes, making the learned representation compact but not explicitly controlling how many atoms are activated across the full dataset. As a result, the regularization weights that govern sparsity and dictionary usage are typically selected as external hyperparameters.

Structured sparse modeling extends this optimization view by introducing group or joint sparsity constraints. Group Lasso and simultaneous sparse approximation encourage structured selection across predefined groups or multiple related signals~\cite{yuan2006model,tropp2006algorithms}. Bayesian group-sparse modeling provides probabilistic interpretations of group selection~\cite{babacan2014bayesian}, and recent dictionary-coding work further explores sparsity and class-sparsity priors for discarding irrelevant dictionary components~\cite{bocchinfuso2023bayesian}. These works are closely related in that they impose structure beyond element-wise sparsity. The constraint considered here is positioned differently: it is atom-wise but defined across samples, so that the row of coefficients associated with an atom determines whether that atom contributes to the effective dictionary size.

PADL is designed to connect these two perspectives. On the optimization side, it gives a tractable objective with reconstruction, coefficient sparsity, and cross-sample atom activation terms. On the probabilistic side, it admits an auxiliary-latent MAP formulation in which the activation variable induces the row-wise regularizer without gating the reconstruction likelihood. This connection allows us to retain the computational form of regularized dictionary learning while deriving quantities that are more naturally obtained from a probabilistic view, including expected activated dictionary size, generalization bounds, and a data-driven activation-threshold rule.
\begin{table*}[t]
\centering
\footnotesize
\setlength{\tabcolsep}{4pt}
\renewcommand{\arraystretch}{1.15}
\caption{High-level comparison of representative dictionary-learning formulations.}
\label{tab:related_work_comparison}
\begin{tabular}{p{2cm} p{4cm} p{6cm}}
\toprule
Perspective & Representative works & Relation to PADL \\
\midrule
Optimization-based
& Sparse coding, K-SVD, online DL~\cite{olshausen1997sparse,aharon2006k,mairal2009online}
& Optimizes reconstruction with coefficient sparsity; PADL further regularizes atom activation across samples. \\
\midrule
Probabilistic
& Bayesian sparse coding, IBP/Beta-process DL, binary/discrete sparse coding~\cite{paisley2009nonparametric,zhou2009non,dang2017indian,dang2018towards,henniges2010binary,exarchakis2017discrete}
& Provides generative views of sparsity and dictionary size; PADL uses an auxiliary activation variable to induce a row-wise MAP regularizer. \\
\midrule
Global-constraint
& Group Lasso, simultaneous sparse approximation, Bayesian group sparsity~\cite{yuan2006model,tropp2006algorithms,babacan2014bayesian,bocchinfuso2023bayesian}
& Models group or shared-support structure; PADL focuses on atom-wise activation across the dataset. \\
\midrule
\textbf{PADL}
& This work
& Connects a storage-aware optimization objective with an auxiliary-latent probabilistic formulation. \\
\bottomrule
\end{tabular}
\end{table*}

\subsection{Compositional Learning Methods on VLM}

As a broader interpretation of dictionary learning, recent work on VLMs has increasingly emphasized compositional representations, where complex visual or multimodal inputs are decomposed into a small set of reusable primitives. UniCode~\cite{zheng2024unicode} proposes a unified codebook that maps heterogeneous modalities into a shared discrete space, enabling cross-modal compositionality and efficient reuse. SAVE~\cite{park2025save} leverages sparse autoencoders to extract semantically meaningful visual factors, showing that structured sparsity can mitigate object hallucination by enforcing interpretable latent components. HybridToken-VLM and HTC-VLM~\cite{zhang2025hybridtoken} compress visual prefixes through hybrid token representations, combining semantic abstraction with token-level efficiency. These methods demonstrate that replacing dense visual streams with structured, compositional units can significantly improve both efficiency and robustness in VLMs. However, existing approaches typically introduce such structure heuristically or implicitly, without a principled mechanism to control the tradeoff among representation sparsity, dictionary size, and reconstruction fidelity. Our work complements these efforts by providing a probabilistically grounded dictionary learning framework that makes this tradeoff explicit and optimizable, enabling compositional compression in VLMs with theoretical guarantees.

\section{Proof of MAP Objective}
\label{ap:map_obj}
This appendix provides the full derivation of the MAP objective in
Proposition~\ref{prop:map_objective}. We show that, under the
probabilistic model, maximum a posteriori (MAP)
estimation of $(D,R)$ yields the deterministic objective consisting of a
quadratic reconstruction term, an $\ell_1$ penalty on coefficients, and
a row-wise $\ell_\infty$ activation penalty.

\subsection{Posterior Energy}

Recall the joint distribution
\[
P(X,R,z;D)
=
\prod_{i=1}^N P(X_i\mid R_i;D)
\cdot
\prod_{k=1}^K\prod_{i=1}^N p(R_{ki})
\cdot
\prod_{k=1}^K P(z_k\mid r_k),
\]
where $r_k=(R_{k1},\dots,R_{kN})$ is the $k$-th row of $R$.
Taking the negative log-posterior and using the inactive event $z_k=0$
for the storage term gives
\begin{align}
-\log P(R,z \mid X; D)
&=
\sum_{i=1}^N \frac{1}{2}\|X_i - D R_i\|_2^2
-
\sum_{k,i} \log p(R_{k i})
-
\sum_{k} \log P(z_k=0 \mid r_k)
+ C,
\label{eq:post_energy}
\end{align}
where $C$ is a constant independent of $(R,z)$.

In the following, we analyze the two prior-induced terms and show that
they correspond to $\ell_1$ and $\ell_\infty$-type regularization.

\subsection{Beta Prior Induces $\ell_1$ Regularization}

\begin{lemma}[Beta Prior and $\ell_1$ Regularization]
\label{lem:beta_l1_app}
Suppose $R_{k i} \sim \mathrm{Beta}(1,\beta)$ with $\beta>1$ and that
$R_{k i}$ are in the small-coefficient regime ($R_{k i}\ll 1$). Then
\begin{equation}
-\sum_{k,i} \log p(R_{k i})
\;\approx\;
\text{\emph{const}} + (\beta-1)\,\|R\|_1,
\end{equation}
where $\|R\|_1=\sum_{k,i}|R_{k i}|$ and the constant term does not depend
on $R$.
\end{lemma}

\begin{proof}
For each coefficient,
\[
p(R_{k i})=\beta(1-R_{k i})^{\beta-1},
\]
and hence
\[
-\log p(R_{k i})
=
-\log\beta-(\beta-1)\log(1-R_{k i}).
\]
For $R_{k i}\ll 1$, we use the first-order Taylor approximation
$\log(1-R_{k i})\approx -R_{k i}$, yielding
\[
-\log p(R_{k i})
\approx
-\log\beta+(\beta-1)R_{k i}.
\]
Summing over all $(k,i)$ gives
\begin{align*}
-\sum_{k,i}\log p(R_{k i})
&\approx
-KN\log\beta
+
(\beta-1)\sum_{k,i}R_{k i} \\
&=
\text{const}+(\beta-1)\|R\|_1,
\end{align*}
which completes the proof.
\end{proof}

Lemma~\ref{lem:beta_l1_app} shows that the Beta prior asymptotically
induces an $\ell_1$ penalty with weight $\lambda_1=\beta-1$, promoting
per-sample sparsity.

\subsection{Activation Prior Induces a Row-wise $\ell_\infty$ Barrier}

Recall the activation model
\[
z_k \mid r_k \sim \operatorname{Bernoulli}
\left(
\frac{1}{1+\exp(-\gamma(\max_i R_{k i}-\delta))}
\right).
\]
Define the softplus function
\[
\phi_\gamma(u)=\log(1+e^{\gamma u}).
\]

\begin{lemma}[Activation Prior and $\ell_\infty$ Barrier]
\label{lem:act_linf_app}
Up to an additive constant independent of $R$,
\[
-\log P(z_k=0 \mid r_k)
\simeq
\phi_\gamma(\max_i R_{k i}-\delta).
\]
Moreover, as $\gamma\to\infty$, $\phi_\gamma$ converges to a hard barrier:
\[
\lim_{\gamma\to\infty}\phi_\gamma(u)=
\begin{cases}
0, & u<0,\\
\infty, & u>0.
\end{cases}
\]
\end{lemma}

\begin{proof}
We have
\[
P(z_k=1\mid r_k)
=
\frac{1}{1+\exp(-\gamma(\max_i R_{k i}-\delta))},
\]
and hence
\[
P(z_k=0\mid r_k)
=
\frac{1}{1+\exp(\gamma(\max_i R_{k i}-\delta))}.
\]
Therefore,
\[
-\log P(z_k=0\mid r_k)
=
\log\!\big(1+\exp(\gamma(\max_i R_{k i}-\delta))\big)
=
\phi_\gamma(\max_i R_{k i}-\delta).
\]
For $u<0$, $\gamma u\to -\infty$ as $\gamma\to\infty$, so
$\phi_\gamma(u)\to 0$. For $u>0$, $\gamma u\to +\infty$ and
$\phi_\gamma(u)\to\infty$. This establishes the barrier
behavior.
\end{proof}

Lemma~\ref{lem:act_linf_app} shows that the activation prior penalizes the
row-wise peak $\max_i R_{k i}$ and converges to a hard constraint
$\max_i R_{k i}\le\delta$ as $\gamma\to\infty$. This corresponds to a
row-wise $\ell_\infty$ regularization on $R$.

\subsection{Derivation of the MAP Objective}

Substituting Lemmas~\ref{lem:beta_l1_app} and~\ref{lem:act_linf_app} into
the posterior energy~\ref{eq:post_energy}, and setting the noise scale
to $\sigma=1$ by normalization, we obtain the MAP objective
\[
\min_{D,R}
\;
\|X-DR\|_F^2
+
(\beta-1)\|R\|_1
+
\sum_{k=1}^K
\phi_\gamma(\max_i R_{k i}-\delta).
\]
Writing $\lambda_1=\beta-1$ and introducing a scaling constant
$\lambda_2>0$ for the activation prior yields
\[
\min_{D,R}
\;
\|X-DR\|_F^2
+
\lambda_1\|R\|_1
+
\lambda_2\sum_{k=1}^K
\phi_\gamma(\max_i R_{k i}-\delta),
\]
which corresponds to Eq.~\ref{eq:novel-obj-main} in the main text. The soft barrier penalizes the row-wise peak $\max_iR_{ki}$ and, for large $\gamma$, approaches a hard threshold constraint at $\delta$. The computational objective in Eq.~\ref{eq:novel-obj-hard} uses the row-wise $\ell_\infty$ term as a convex activation surrogate for this peak-activation cost, thereby completing the proof of Proposition~\ref{prop:map_objective}.

\section{Technical Lemmas for Tradeoff Characterization}

This appendix provides the technical foundations for the tradeoff
analysis. We establish sufficient conditions under which
non-essential atoms are suppressed, derive a high-probability upper
bound on the aggregate activation mass of an atom, and combine these
results to obtain a reconstruction error bound.

Throughout this section, let $r_k=(R_{k1},\dots,R_{kN})$ denote the
coefficient row corresponding to atom $d_k$, and define the objective
\[
f(D,R)
=
\|X - D R\|_F^2
+
\lambda_1 \|R\|_1
+
\lambda_2\sum_{k=1}^K \|r_k\|_\infty ,
\]
where $\|r_k\|_\infty=\max_i |R_{k i}|$.

\subsection{Sufficient Condition for Atom Suppression}
\label{ap:sufficient_cond}

\begin{lemma}[Sufficient Condition for Atom Suppression]
\label{lem:suppress_app}
Assume that a subset of atoms $S$ satisfies
\[
\|X - D_S R_S\|_F^2 \le \varepsilon,
\]
where $D_S$ and $R_S$ denote the dictionary and coefficients restricted
to $S$. For an atom $k\notin S$, let $r_k$ be its coefficient row and
define the residual $e = X - D_S R_S$. Suppose $\|d_k\|_2=1$ and that
\[
|e_i^\top d_k| \le \eta
\quad\text{for all } i.
\]
If $r_k\neq 0$ and
\begin{equation}
\label{eq:suppress_cond_app}
\lambda_1
+
\lambda_2\frac{\|r_k\|_\infty}{\|r_k\|_1}
\;>\;
2\eta,
\end{equation}
then setting $r_k=0$ strictly decreases the objective, i.e.,
$f(D,R_{-k})-f(D,R)<0$. In particular, for an atom whose peak activation
satisfies $\|r_k\|_\infty\ge\delta$, it is sufficient that
\begin{equation}
\label{eq:suppress_cond_delta_app}
\lambda_1
+
\lambda_2\frac{\delta}{\|r_k\|_1}
\;>\;
2\eta.
\end{equation}
\end{lemma}

\begin{proof}
Let $R_{-k}$ denote the coefficient matrix obtained by zeroing the
$k$-th row of $R$, and define $\Delta f := f(D,R_{-k}) - f(D,R)$.

\paragraph{Reconstruction term.}
Using $X-DR = e - d_k r_k^\top$ and $\|d_k\|_2=1$, we have
\[
\|X-DR\|_F^2
=
\|e\|_F^2
-2\sum_{i=1}^N e_i^\top d_k\,R_{k i}
+
\|r_k\|_2^2.
\]
Therefore,
\[
\Delta_{\mathrm{rec}}
=
\|e\|_F^2-\|X-DR\|_F^2
=
2\sum_{i=1}^N e_i^\top d_k\,R_{k i}
-
\|r_k\|_2^2
\le
2\eta\,\|r_k\|_1,
\]
where the negative quadratic term has been dropped to obtain a valid
upper bound.

\paragraph{Regularization term.}
Zeroing $r_k$ reduces the $\ell_1$ penalty by $\lambda_1\|r_k\|_1$ and
reduces the row-wise $\ell_\infty$ penalty by
$\lambda_2\|r_k\|_\infty$. Hence
\[
\Delta_{\mathrm{reg}}
=
-\lambda_1\|r_k\|_1
-
\lambda_2\|r_k\|_\infty.
\]

\paragraph{Combining terms.}
Combining the two bounds gives
\[
\Delta f
\le
(2\eta-\lambda_1)\|r_k\|_1
-
\lambda_2\|r_k\|_\infty.
\]
Thus $\Delta f<0$ whenever Eq.~\ref{eq:suppress_cond_app} holds. If
$\|r_k\|_\infty\ge\delta$, then
$\|r_k\|_\infty/\|r_k\|_1\ge\delta/\|r_k\|_1$, and the simpler
threshold-dependent condition in Eq.~\ref{eq:suppress_cond_delta_app}
implies Eq.~\ref{eq:suppress_cond_app}.
\end{proof}

\subsection{Expected Dictionary Activation Size}
\label{ap:exp_s}
For each atom $k$, define the indicator of being active under threshold
$\delta$:
\[
\mathbf{1}_k(\delta)
\;:=\;
\mathbb{I}\!\left\{\max_{1\le i\le N} R_{k i}>\delta\right\}.
\]
Then the active set size satisfies
\[
|S(\delta)|=\sum_{k=1}^K \mathbf{1}_k(\delta),
\qquad
\Rightarrow\quad
\mathbb{E}|S(\delta)|
=\sum_{k=1}^K \mathbb{E}\,\mathbf{1}_k(\delta).
\]
By identical distribution across $k$, it suffices to compute
$\mathbb{E}\,\mathbf{1}_1(\delta)=\Pr(\max_i R_{1i}>\delta)$.

Since $\{R_{1i}\}_{i=1}^N$ are i.i.d.\ with CDF $F$, we have
\[
\Pr\!\left(\max_{1\le i\le N} R_{1i}\le \delta\right)
=
\prod_{i=1}^N \Pr(R_{1i}\le \delta)
=
F(\delta)^N.
\]
Therefore,
\[
\Pr\!\left(\max_{1\le i\le N} R_{1i}>\delta\right)
=
1-F(\delta)^N,
\]
and hence
\[
\mathbb{E}|S(\delta)|
=
\sum_{k=1}^K \Pr\!\left(\max_{1\le i\le N} R_{k i}>\delta\right)
=
K\bigl(1-F(\delta)^N\bigr).
\]
Substituting $F(\delta)=1-(1-\delta)^\beta$ yields
\[
\mathbb{E}|S(\delta)|
=
K\left[1-\bigl(1-(1-\delta)^\beta\bigr)^N\right],
\]
which proves~\ref{eq:ES_delta_app}.

To show monotonicity and smoothness, note that for $\beta>0$,
$F(\delta)$ is smooth and strictly increasing on $(0,1)$ with density
\[
f(\delta)=F'(\delta)=\beta(1-\delta)^{\beta-1}>0.
\]
Thus, $\mathbb{E}|S(\delta)|$ is smooth on $(0,1)$ and
\[
\frac{d}{d\delta}\mathbb{E}|S(\delta)|
=
-K\,N\,F(\delta)^{N-1} f(\delta)
<0,
\]
so it is strictly decreasing in $\delta$.

Finally, the endpoint values follow immediately:
$F(0)=0 \Rightarrow \mathbb{E}|S(0)|=K(1-0)=K$, and
$F(1)=1 \Rightarrow \mathbb{E}|S(1)|=K(1-1)=0$.

\subsection{High-Probability Reconstruction Error Bound}
\label{ap:err_bound}
\begin{lemma}[High-Probability Upper Bound on $\|r_k\|_1$]
\label{lem:l1_upper_app}
Assume $R_{k i} \overset{\mathrm{i.i.d.}}{\sim} \mathrm{Beta}(1,\beta)$
with $\beta>1$. Then for any $\tau\in(0,1)$, with probability at least
$1-\tau$,
\begin{equation}
\label{eq:l1_upper_app}
\|r_k\|_1
\le
U(N,\beta,\tau)
:=
\frac{N}{1+\beta}
+
\sqrt{\frac{N}{2}\log\frac{1}{\tau}}.
\end{equation}
\end{lemma}

\begin{proof}
Each $R_{k i}$ lies in $[0,1]$ and satisfies
$\mathbb{E}[R_{k i}] = \frac{1}{1+\beta}$. Let
$S_k=\sum_{i=1}^N R_{k i}=\|r_k\|_1$. By Hoeffding's inequality,
\[
\Pr(S_k - \mathbb{E}[S_k] \ge t)
\le
\exp\!\left(-\frac{2t^2}{N}\right).
\]
Setting
$t=\sqrt{\frac{N}{2}\log\frac{1}{\tau}}$ yields
\[
\Pr\!\left(
S_k \ge \frac{N}{1+\beta} + \sqrt{\frac{N}{2}\log\frac{1}{\tau}}
\right)
\le \tau,
\]
which proves the claim.
\end{proof}

\begin{lemma}[High-Probability Reconstruction Bound]
\label{lem:err_bound_app}
Let $e=X-D_S R_S$ be the residual using the active atom set $S$, and
define
\[
\eta := \max_{i,k} |e_i^\top d_k|.
\]
Assume that every non-essential active candidate $k\notin S$ satisfies
$\|r_k\|_\infty\ge\delta$. Assume also that the active sub-dictionary
satisfies the restricted active-span condition
\begin{equation}
\label{eq:active_span_re_app}
\|D_S^\top e_i\|_2^2
\ge
\alpha_S\|e_i\|_2^2
\qquad\text{for all }i,
\end{equation}
for some $\alpha_S>0$. Then, for any $\tau\in(0,1)$, with probability at
least $1-\tau$,a sufficient condition for suppressing each fixed non-essential atom is
\begin{equation}
\label{eq:eta_bound_app}
2\eta
\le
\lambda_1
+
\frac{\lambda_2\delta}{U(N,\beta,\tau)}.
\end{equation}
Under this condition, the reconstruction error satisfies
\begin{equation}
\label{eq:err_bound_app}
\|e\|_F^2
\le
\frac{N}{4}\,
\frac{|S|}{\alpha_S}
\left(
\lambda_1
+
\frac{\lambda_2\delta}{U(N,\beta,\tau)}
\right)^2
\quad\text{with probability at least }1-\tau.
\end{equation}
\end{lemma}

\begin{proof}
From Lemma~\ref{lem:suppress_app}, if $\|r_k\|_\infty\ge\delta$, a
sufficient condition for suppressing a non-essential atom $k\notin S$ is
\[
2\eta
\le
\lambda_1
+
\lambda_2\frac{\delta}{\|r_k\|_1}.
\]
By Lemma~\ref{lem:l1_upper_app}, with probability at least $1-\tau$,
$\|r_k\|_1\le U(N,\beta,\tau)$. Therefore
$\delta/\|r_k\|_1\ge \delta/U(N,\beta,\tau)$, and
Eq.~\ref{eq:eta_bound_app} is a sufficient condition.

For the geometric part, Eq.~\ref{eq:active_span_re_app} gives
\[
\alpha_S\|e_i\|_2^2
\le
\|D_S^\top e_i\|_2^2
=
\sum_{k\in S}(e_i^\top d_k)^2.
\]
Since $(e_i^\top d_k)^2 \le \eta^2$, we have
$\sum_{k\in S}(e_i^\top d_k)^2 \le |S|\eta^2$, hence
\[
\|e_i\|_2^2 \le \frac{|S|}{\alpha_S}\eta^2.
\]
Summing over all $i$ yields
\[
\|e\|_F^2
\le
N\,\frac{|S|}{\alpha_S}\,\eta^2.
\]
Substituting the bound on $\eta$ from~\ref{eq:eta_bound_app} gives
\ref{eq:err_bound_app}.
\end{proof}

\begin{remark}[Geometry of the active-span condition]
The condition in Eq.~\ref{eq:active_span_re_app} avoids requiring
$D_SD_S^\top$ to be full rank. If residuals are restricted to
$\operatorname{span}(D_S)$ and $D_S^\top D_S\succeq \alpha_S I$, then
Eq.~\ref{eq:active_span_re_app} holds. For arbitrary residuals, the same
argument controls the component projected onto the active span.
\end{remark}

\section{Scaling Behavior and Interpretation}
Recall from Lemma~\ref{lem:err_bound_app} that, with probability at least
$1-\tau$, the reconstruction error satisfies
\begin{equation}
\label{eq:appC_main}
\|e\|_F^2
\le
\frac{N}{4}\,
\frac{|S|}{\alpha_S}
\left(
\lambda_1
+
\frac{\lambda_2\delta}{U(N,\beta,\tau)}
\right)^2 ,
\end{equation}
where $S$ is the active atom set, $|S|$ its cardinality, and $\alpha_S$
is the restricted active-span constant from Eq.~\ref{eq:active_span_re_app}.

\subsection{Dependence on the Number of Samples}

For fixed $\beta$ and confidence level $\tau$,
\[
U(N,\beta,\tau)
=
\frac{N}{1+\beta}
+
\sqrt{\frac{N}{2}\log\frac{1}{\tau}}
\]
grows linearly with $N$ as $N\to\infty$. Consequently, the contribution
of the activation term inside the parentheses in
Eq.~\ref{eq:appC_main} decays as $\mathcal{O}(1/N)$, and the bound
asymptotically simplifies to
\begin{equation}
\frac{\|e\|_F^2}{N}
=
\mathcal{O}\!\left(
\frac{|S|}{\alpha_S}\,\lambda_1^2
\right).
\end{equation}
Thus, the reconstruction error \emph{per sample} remains bounded in the
large-sample regime.

Intuitively, as more visual samples become available, the aggregate
activation mass $\|r_k\|_1$ of each atom concentrates around its
expectation under the Beta prior. The threshold contribution
$\lambda_2\delta/U(N,\beta,\tau)$ therefore vanishes in the large-sample
regime, while the active set size is still controlled by the threshold
through $\mathbb{E}|S(\delta)|$.

\subsection{Roles of $\ell_1$ and $\ell_\infty$}

The two regularizers in the objective contribute in complementary ways.

The $\ell_1$ term determines the asymptotic error floor through
$\lambda_1$. It controls the average magnitude and count of coefficients
within each sample, thereby governing per-sample representation
complexity. In the large-$N$ regime, $\lambda_1$ becomes the dominant
factor shaping the reconstruction error.

In contrast, the row-wise $\ell_\infty$ term enters the suppression
condition through
\[
\frac{\|r_k\|_\infty}{\|r_k\|_1},
\]
and, for atoms whose peak exceeds the threshold, through the conservative
quantity $\delta/U(N,\beta,\tau)$. This ratio is large for atoms whose
usage is concentrated in a small number of samples, so such atoms are
more easily suppressed when they are weakly aligned with the residual.
Conversely, atoms that contribute frequently across samples incur only a
vanishing activation penalty as $N$ grows.

This mechanism formalizes the intended behavior of the combined
$\ell_1$--$\ell_\infty$ regularization: it favors atoms that provide
stable, compositional structure while eliminating peak-only or
idiosyncratic components. In the context of visual representations, the
model retains primitives that recur across many samples and discards
those that encode transient or sample-specific artifacts.

\subsection{Dependence on Dictionary Geometry}

The factor $|S|/\alpha_S$ in Eq.~\ref{eq:appC_main} captures the
effective model complexity. The cardinality $|S|$ measures how many
atoms are actually used, while $\alpha_S$ quantifies how well the active
span captures the residual directions being bounded.

Even with a large nominal dictionary, the bound depends only on the
\emph{active} subset and its restricted active-span conditioning. Poorly
conditioned or overly redundant selections (small $\alpha_S$) can
inflate the bound, whereas compact and well-spread subsets yield tighter
guarantees. This reveals that reconstruction accuracy is governed jointly
by statistical regularization and the geometry of the selected
representation.

\subsection{Optimal Activation Threshold $\delta^\star$}
\label{ap:opt_delta}
We can further obtain an explicit characterization of the optimal
activation threshold $\delta$ by minimizing a proxy of the theoretical
upper bound. The key observation is that $\delta$ controls the expected
active dictionary size through the Beta prior, while simultaneously
entering the reconstruction bound through the activation penalty term.

\paragraph{Active set size as a function of $\delta$.}
Define the active atom set induced by thresholding
\[
S(\delta) := \{ k \in [K] : \max_{1\le i\le N} R_{k i} > \delta \}.
\]
Under $R_{k i}\overset{i.i.d.}{\sim}\mathrm{Beta}(1,\beta)$, let
\[
F(\delta) := \Pr(R_{k i}\le \delta) = 1-(1-\delta)^\beta,
\qquad
f(\delta) := F'(\delta)=\beta(1-\delta)^{\beta-1}.
\]
Then
\begin{equation}
\label{eq:ES_delta_app}
\mathbb{E}|S(\delta)|
=
K\Big(1 - F(\delta)^N\Big),
\end{equation}
since $\Pr(\max_i R_{k i}\le \delta)=F(\delta)^N$.

\paragraph{A bound-driven objective.}
From Lemma~\ref{lem:err_bound_app}, the reconstruction error bound takes
the form
\[
\|e\|_F^2
\;\lesssim\;
\frac{N}{4}\,\frac{|S|}{\alpha_S}
\left(\lambda_1+
\frac{\lambda_2\delta}{U(N,\beta,\tau)}\right)^2.
\]
Treating $\alpha_S$ as approximately constant with respect to $\delta$
(standard in such proxy optimization), and replacing $|S|$ by its
expectation in~\ref{eq:ES_delta_app}, we minimize the scalar objective
\begin{equation}
\label{eq:phi_delta}
\Phi(\delta)
:=
\Big(1 - F(\delta)^N\Big)
\left(\lambda_1+
\frac{\lambda_2\delta}{U(N,\beta,\tau)}\right)^2,
\qquad
\delta\in[0,1].
\end{equation}
This objective explicitly captures the sparsity--storage--accuracy
tradeoff: $1-F(\delta)^N$ decreases with $\delta$ (fewer active atoms,
lower storage), while the factor
$\lambda_1+\lambda_2\delta/U(\cdot)$ increases with $\delta$ (stronger
activation constraint, potentially higher error).

\paragraph{Closed-form optimality condition.}
A stationary point $\delta^\star\in(0,1)$ of~\ref{eq:phi_delta} satisfies
\begin{equation}
\label{eq:delta_stationary_app}
\frac{N\,F(\delta)^{N-1}\,f(\delta)}{1-F(\delta)^N}
=
\frac{2\lambda_2}{U(N,\beta,\tau)\left(\lambda_1+
\frac{\lambda_2\delta}{U(N,\beta,\tau)}\right)}.
\end{equation}
The left-hand side is determined solely by the Beta prior and sample
count $N$, while the right-hand side depends on the regularization scale
$\lambda_1$, activation weight $\lambda_2$, and concentration term
$U(N,\beta,\tau)$. In practice, $\delta^\star$ can be obtained
efficiently by one-dimensional root finding (e.g., bisection or Newton's
method), yielding a stable and deterministic hyperparameter selection
rule.
\begin{remark}[Non-degenerate regime]
The bound in Proposition~\ref{prop:err_bound} is derived under the
restricted active-span condition in Eq.~\ref{eq:active_span_re_app}. When
$\delta\to1$, the active set $S(\delta)$ becomes empty and this condition
is violated, so the bound no longer applies. Therefore, the endpoint
$\delta=1$ corresponds to a degenerate regime and is excluded from the
feasible domain. The meaningful optimum lies in the interior
$\delta\in(0,1)$, where the stationary condition in
Eq.~\ref{eq:delta_stationary_app} characterizes the trade-off between
sparsity, storage, and accuracy.
\end{remark}

\subsection{Interpretation from an MDL Perspective}

From a Minimum Description Length (MDL) viewpoint, the objective
\[
\|X - D R\|_F^2
+
\lambda_1 \|R\|_1
+
\sum_{k=1}^K \|r_k\|_\infty
\]
can be interpreted as trading off three code lengths: the cost of
encoding the residual, the cost of encoding sparse coefficients, and
the cost of instantiating dictionary atoms.

The $\ell_1$ term penalizes the number and magnitude of nonzero
coefficients, corresponding to the description length of per-sample codes. The row-wise $\ell_\infty$ term penalizes whether an atom is used
at all, corresponding to the cost of introducing a new primitive into
the model. In this sense, the activation penalty performs model
selection over atoms, while the sparsity penalty performs compression
within the selected model.

The theory therefore shows that the PADL formulation can be interpreted as
Bayesian model selection under an auxiliary activation prior.
\section{VLM Inference Compression}
\label{app:vlm_experiments}

This section provides the detailed VLM formulation, compression protocol, and
supplementary validation for PADL-based visual-representation compression. The
main image and video results are reported in the main text, while the appendix
focuses on implementation details and threshold validation.

\subsection{VLM Integration}
\label{app:vlm_implementation}

The parsimoniously activated dictionary learning paradigm can be naturally integrated into VLM inference as a structured visual compression module. In modern multimodal systems, images and videos are encoded into long sequences of visual tokens that often contain substantial redundancy. These dense prefixes dominate both computation and communication cost, especially in edge-cloud deployment scenarios, where visual features must be transmitted to a remote language model.

Consider a standard VLM pipeline. An image
$\mathbf{I}\in\mathbb{R}^{C\times H\times W}$ or a video
$\mathbf{I}_{\text{vid}}\in\mathbb{R}^{N_t\times C\times H\times W}$ is
processed by a vision encoder that partitions each frame into
$N_s=\frac{H}{p}\times\frac{W}{p}$ non-overlapping patches and maps each
patch to a $d$-dimensional embedding. For frame $i$, this yields
\[
\mathbf{Z}_V^{(i)} = [\mathbf{z}_{i,1},\ldots,\mathbf{z}_{i,N_s}]
\in \mathbb{R}^{N_s\times d},
\quad \mathbf{z}_{i,j}\in\mathbb{R}^d.
\]
Aggregating all frames gives a visual token matrix
\[
\mathbf{Z}_V \in \mathbb{R}^{N_V\times d},
\qquad N_V = N_t N_s,
\]
which is concatenated with the textual embedding
$\mathbf{Z}_T\in\mathbb{R}^{L_T\times d}$ and fed into the LLM for
autoregressive generation.

Our method inserts a dictionary-based compression layer on
$\mathbf{Z}_V$. We learn a visual dictionary
$\mathbf{D}_V\in\mathbb{R}^{d\times K}$ and compute a sparse
representation
\[
\mathbf{Z}_V \approx (\mathbf{D}_V \mathbf{R}_V)^T,
\qquad
\mathbf{R}_V \in \mathbb{R}^{K\times N_V},
\]
where each column of $\mathbf{R}_V$ is sparse and only a small subset of atoms in $\mathbf{D}_V$ are globally activated. In an edge-cloud setting, the sparse codes can be quantized and transmitted:
\[
\widetilde{\mathbf{R}}_V = \operatorname{Quantize}(\mathbf{R}_V),
\]
together with the activated sub-dictionary $\mathbf{D}_S \subset \mathbf{D}_V$. As shown in Figure~\ref{fig:imp_vlm}, on the cloud side, visual tokens can be reconstructed by
\begin{equation}
    \mathbf{Z}_V^\ast = (\mathbf{D}_S \widetilde{\mathbf{R}}_V)^T, \nonumber
\end{equation}
and then concatenated with $\mathbf{Z}_T$ for standard VLM inference. Alternatively, a lightweight adapter $f_a$ can be trained to directly process the sparse representation:
\[
\widetilde{\mathbf{Z}}_V^\ast = f_a(\widetilde{\mathbf{R}}_V, \mathbf{D}_S),
\]
bypassing explicit reconstruction. This enables the language model to operate on compressed visual features, yielding further acceleration. In both modes, our framework reduces the visual representation size and controls the number of stored primitives through a principled sparsity-storage-accuracy tradeoff.

\subsection{Compression Setting and Evaluation Protocol}
\label{app:vlm_setup}

We evaluate PADL as a sparse visual-representation compression module. In this setting, the visual features produced by the vision encoder are projected onto a learned dictionary, and the resulting sparse codes are used as compact visual representations. The reported compression ratio refers to the reduction in visual representation size induced by dimensional compression and quantization. Specifically, for image inputs we apply $4\times$ dimensional compression together with $2\times$ quantization, resulting in an overall $8\times$ representation-size reduction. For video inputs, we apply $8\times$ dimensional compression together with $2\times$ quantization, resulting in an overall $16\times$ representation-size reduction.

For image VQA experiments, we use LLaVA-1.5~\cite{liu2024improved} as the backbone and evaluate on VQAv2~\cite{goyal2017making}, ScienceQA-IMG~\cite{lu2022learn}, TextVQA~\cite{singh2019towards}, POPE~\cite{li2023evaluating}, MME~\cite{fu2023mme}, and MMBench~\cite{liu2024mmbench}. The dictionary contains 1024 atoms. We compare the original LLaVA-1.5 model with LLaVA-PruMerge~\cite{shang2024llava}, FreePruner~\cite{xu2024freepruner}, PADL without adapter fine-tuning, and PADL with a lightweight adapter. We denote the version without adapter fine-tuning as \textbf{PADL}, and the version with adapter fine-tuning as \textbf{PADL+}. The adapter is trained on the LLaVA-1.5 instruction-tuning dataset.

For video VQA experiments, we use Video-LLaVA as the backbone and evaluate on MSVD-QA~\cite{chen2023x}, MSRVTT-QA~\cite{xu2016msr}, ActivityNet-QA~\cite{yu2019activitynet}, and TGIF-QA~\cite{jang2017tgif}. The adapter is fine-tuned on the Valley dataset~\cite{luo2023valley}. The evaluation follows the same protocol as prior video VQA work, reporting both accuracy and GPT-based answer-quality scores where applicable.

\subsection{Computational Overhead}
\label{app:vlm_cost}

Table~\ref{tab:vlm_cost_app} reports the parameter and FLOP overhead introduced by the PADL-based VLM compression module. The dictionary encoding and adapter introduce additional computation, but this overhead is substantially smaller than the downstream LLM-side computation saved by compressing visual representations. This supports the practical motivation of applying PADL to VLM inference, where the visual prefix can contribute significantly to the total computational cost. In the direct-adapter mode, the adapter maps the compressed sparse representation to a shorter visual prefix before LLM decoding, which accounts for the reported LLM-side FLOP saving.

\begin{table}[h]
\centering
\small
\setlength{\tabcolsep}{6pt}
\renewcommand{\arraystretch}{1.15}
\caption{Parameter and FLOP overhead of the PADL-based VLM compression module. The LLM-side computation saving is measured under the video-input setting.}
\label{tab:vlm_cost_app}
\begin{tabular}{lcc}
\toprule
Item & Parameters & FLOPs \\
\midrule
Dictionary encoding & 8.91M & +1.973T \\
Adapter & 38.05M & +13.0G \\
LLM-side computation saving & -- & $-25.64$T \\
\bottomrule
\end{tabular}
\end{table}

\subsection{Threshold Validation on VLM Compression}
\label{app:vlm_threshold}

To examine whether the threshold predicted by our theory remains meaningful in the VLM setting, we perform a threshold sweep on VQAv2 using PADL without adapter fine-tuning. We vary the activation threshold $\delta$ from 0.38 to 0.46 with step size 0.02. The theoretically predicted optimal threshold is $\delta=0.412$, which is close to the empirical optimum at $\delta=0.42$.

\begin{table}[h]
\centering
\small
\setlength{\tabcolsep}{8pt}
\renewcommand{\arraystretch}{1.15}
\caption{Threshold validation for PADL-based VLM compression on VQAv2 without adapter fine-tuning. The theoretically predicted threshold $\delta=0.412$ closely matches the empirical optimum around $\delta=0.42$.}
\label{tab:vlm_delta_sweep}
\begin{tabular}{c|ccccc}
\toprule
$\delta$ & 0.38 & 0.40 & 0.42 & 0.44 & 0.46 \\
\midrule
VQAv2 Accuracy & 71.5 & 74.8 & \textbf{77.0} & 76.8 & 76.7 \\
\bottomrule
\end{tabular}
\end{table}

These results suggest that the proposed data-driven threshold estimation can provide a useful operating point even in the VLM compression setting. We emphasize that this experiment is not intended to replace the controlled reconstruction experiments in the main text, but rather to show that the threshold-selection principle transfers to a larger multimodal application scenario.

\subsection{Main VLM Results}
\label{app:main_vlm_results}

The main image and video VQA results are reported in
Tables~\ref{tab:img_qa_app} and~\ref{tab:video_qa_app}. This appendix focuses
on the detailed compression setup, computational overhead, and threshold
validation.

\subsection{Discussion}
\label{app:vlm_discussion}

The VLM experiments support three observations. First, PADL can be used as a plug-in visual compression module for both image and video VLMs. Second, the predicted activation threshold provides a useful operating point in the VLM setting, as shown by the VQAv2 threshold sweep in Table~\ref{tab:vlm_delta_sweep}. Third, adapter fine-tuning can further align compressed sparse representations with the downstream language model, improving the performance of PADL+ over the no-adapter variant.

We emphasize that these experiments are not intended to claim that PADL universally dominates all VLM token-pruning or token-merging methods. Instead, they show that this SSA-based dictionary learning framework can be integrated into multimodal inference pipelines and can achieve favorable compression--performance tradeoffs. The main theoretical and empirical claims of the paper are supported by the controlled reconstruction experiments in the main text, where reconstruction accuracy, coefficient sparsity, and activated dictionary size can be directly measured.

\section{Additional Dictionary Learning Experiments}
\label{app:additional_dl_experiments}

This section reports additional dictionary-learning experiments that further support the empirical claims in the main paper. These experiments are designed to evaluate modern sparse-representation baselines, optimization stability, robustness to warm-up length and nominal dictionary size, mini-batch approximation, post-training pruning, and the effect of patch-boundary artifacts on SSIM. They are placed in the appendix because the main text focuses on the core sparsity--storage--accuracy (SSA) tradeoff and the corresponding theoretical analysis.

\subsection{Comparison with Recent Sparse-Representation Baselines}
\label{app:modern_baselines}

In addition to classical dictionary-learning and Bayesian sparse coding baselines, we compare PADL with recent sparse-representation methods, including CRsAE and JumpReLU SAE. We also report an $L_1$-regularized dictionary learning baseline to isolate the effect of the proposed row-wise activation regularization. All results are averaged over 10 repeated runs.

\begin{table*}[h]
\centering
\small
\setlength{\tabcolsep}{6pt}
\renewcommand{\arraystretch}{1.15}
\caption{Comparison with recent sparse-representation baselines on CIFAR-100 and SVHN. Results are reported as mean and standard deviation over 10 repeated runs. Lower RMSE is better; higher PSNR and SSIM are better.}
\label{tab:app_modern_baselines_std}
\begin{tabular}{llccc}
\toprule
Dataset & Method & RMSE & PSNR & SSIM \\
\midrule
\multirow{4}{*}{CIFAR-100}
& CRsAE~\cite{tolooshams2020deep} & $0.106 \pm 0.007$ & $20.128 \pm 0.561$ & $0.557 \pm 0.035$ \\
& JumpReLU SAE~\cite{rajamanoharan2024jumping} & $0.104 \pm 0.016$ & $19.584 \pm 0.971$ & $0.520 \pm 0.041$ \\
& $L_1$-Regularized DL & $0.112 \pm 0.006$ & $18.980 \pm 0.540$ & $0.545 \pm 0.030$ \\
& \textbf{PADL} & $\mathbf{0.099 \pm 0.009}$ & $\mathbf{20.560 \pm 0.554}$ & $\mathbf{0.568 \pm 0.030}$ \\
\midrule
\multirow{4}{*}{SVHN}
& CRsAE~\cite{tolooshams2020deep} & $0.092 \pm 0.006$ & $20.302 \pm 0.599$ & $0.565 \pm 0.025$ \\
& JumpReLU SAE~\cite{rajamanoharan2024jumping} & $0.091 \pm 0.010$ & $20.008 \pm 1.194$ & $0.518 \pm 0.042$ \\
& $L_1$-Regularized DL & $0.094 \pm 0.012$ & $20.465 \pm 1.031$ & $0.559 \pm 0.037$ \\
& \textbf{PADL} & $\mathbf{0.090 \pm 0.006}$ & $\mathbf{20.880 \pm 1.030}$ & $\mathbf{0.583 \pm 0.037}$ \\
\bottomrule
\end{tabular}
\end{table*}

These results show that PADL is competitive not only with classical sparse coding methods, but also with more recent sparse-representation baselines. The improvement over the $L_1$-regularized baseline suggests that the row-wise activation-aware term provides benefits beyond ordinary element-wise sparsity.

\subsection{Training Time and Activated Dictionary Size}
\label{app:training_time_atoms}

We further compare training time, activated atoms, and reconstruction error on CIFAR-100 using 2000 samples, $8\times 8$ patches, and 400 training iterations. The number of activated atoms measures the effective dictionary size after training.

\begin{table}[h]
\centering
\small
\setlength{\tabcolsep}{7pt}
\renewcommand{\arraystretch}{1.15}
\caption{Training time, activated atoms, and reconstruction error on CIFAR-100. PADL is not the fastest method, but it achieves the lowest RMSE while using the fewest activated atoms among the recent sparse-representation baselines.}
\label{tab:app_training_time_atoms}
\begin{tabular}{lccc}
\toprule
Method & Train Time & Activated Atoms & RMSE \\
\midrule
K-SVD & 661.4s & 20 & 0.152 \\
CRsAE~\cite{tolooshams2020deep} & 32.0s & 56 & 0.106 \\
JumpReLU SAE~\cite{rajamanoharan2024jumping} & 3.1s & 61 & 0.104 \\
\textbf{PADL} & 44.2s & \textbf{19} & \textbf{0.099} \\
\bottomrule
\end{tabular}
\end{table}

The results indicate that PADL remains within a practical training-time range while learning a more compact effective dictionary. Compared with CRsAE and JumpReLU SAE, PADL uses fewer activated atoms and achieves lower reconstruction error. This supports the claim that PADL improves the sparsity--storage--accuracy tradeoff rather than merely improving reconstruction accuracy.

\subsection{Effect of Patch-Boundary Artifacts on SSIM}
\label{app:ssim_patch_artifacts}

The relatively low SSIM values in the main patch-based reconstruction protocol are partly caused by patch-boundary artifacts. To verify this, we conduct an additional experiment on CIFAR-100 using overlapping reconstruction. The overlap reduces stitching artifacts between neighboring patches.

\begin{table}[h]
\centering
\small
\setlength{\tabcolsep}{7pt}
\renewcommand{\arraystretch}{1.15}
\caption{Reconstruction performance on CIFAR-100 with overlapping reconstruction. The improvement in SSIM confirms that the lower SSIM in the original patch-based protocol is partly caused by patch-boundary artifacts.}
\label{tab:app_overlap_reconstruction}
\begin{tabular}{lccc}
\toprule
Method & SSIM & RMSE & PSNR \\
\midrule
CRsAE~\cite{tolooshams2020deep} & 0.8337 & 0.1012 & 23.0179 \\
JumpReLU SAE~\cite{rajamanoharan2024jumping} & 0.7366 & 0.0959 & 23.9402 \\
\textbf{PADL} & \textbf{0.8402} & \textbf{0.0937} & \textbf{24.7478} \\
\bottomrule
\end{tabular}
\end{table}

With overlapping reconstruction, PADL achieves an SSIM of 0.8402, substantially higher than the value obtained under the non-overlapping patch protocol. This suggests that the original SSIM scores should be interpreted together with the patch-based evaluation setting, and that PADL remains competitive when patch-boundary artifacts are mitigated.

\subsection{Optimization Stability}
\label{app:optimization_stability}

Although the PADL objective is non-convex and contains a non-smooth row-wise $\ell_\infty$-type regularization term, the alternating optimization procedure is empirically stable. Table~\ref{tab:app_optimization_terms} reports the evolution of the reconstruction term, the element-wise sparsity term, and the row-wise activation term on CIFAR-100 and SVHN. Each term is averaged by sample size.

\begin{table*}[h]
\centering
\small
\setlength{\tabcolsep}{7pt}
\renewcommand{\arraystretch}{1.15}
\caption{Optimization behavior of the three objective terms on CIFAR-100 and SVHN. The objective terms generally decrease and stabilize during training, suggesting stable empirical optimization behavior.}
\label{tab:app_optimization_terms}
\begin{tabular}{ccccc|cccc}
\toprule
\multirow{2}{*}{Iter}
& \multicolumn{3}{c|}{CIFAR-100}
& \phantom{x}
& \multicolumn{3}{c}{SVHN} \\
\cmidrule(lr){2-4}
\cmidrule(lr){6-8}
& Recon & $L_1$ & $L_\infty$ &&
Recon & $L_1$ & $L_\infty$ \\
\midrule
20  & 0.774834 & 0.252 & 0.748664 && 1.344806 & 0.251 & 0.701239 \\
100 & 0.658850 & 0.186 & 0.701099 && 1.084268 & 0.224 & 0.507329 \\
200 & 0.649933 & 0.165 & 0.594202 && 0.737957 & 0.214 & 0.451594 \\
300 & 0.647558 & 0.151 & 0.518092 && 0.583525 & 0.203 & 0.447624 \\
400 & 0.647077 & 0.141 & 0.507625 && 0.545520 & 0.209 & 0.439489 \\
\bottomrule
\end{tabular}
\end{table*}

The reconstruction and regularization terms show clear descent and stabilization behavior on both datasets. On SVHN, the $L_1$ term slightly increases from iteration 300 to 400, but the overall optimization remains stable and the reconstruction term continues to decrease. These results provide empirical evidence that the non-smooth activation-aware regularization does not cause severe oscillation or divergence under our training settings.

\subsection{Robustness to Warm-Up Length}
\label{app:warmup_sensitivity}

PADL estimates prior-related parameters during a warm-up stage. To evaluate whether the method is sensitive to the exact warm-up length, we vary the number of warm-up steps $t_s$ on CIFAR-100 and report results over 10 repeated runs.

\begin{table}[h]
\centering
\small
\setlength{\tabcolsep}{5pt}
\renewcommand{\arraystretch}{1.15}
\caption{Effect of warm-up steps $t_s$ on CIFAR-100. Results are reported over 10 repeated runs.}
\label{tab:app_warmup_steps}
\begin{tabular}{c|ccccc}
\toprule
$t_s$ & 10 & 30 & 50 & 70 & 100 \\
\midrule
RMSE
& $0.1074 \pm 0.0127$
& $0.1015 \pm 0.0115$
& $0.0994 \pm 0.0093$
& $0.0992 \pm 0.0101$
& $0.0989 \pm 0.0088$ \\
\bottomrule
\end{tabular}
\end{table}

When $t_s$ increases from 10 to 100, RMSE improves from 0.1074 to 0.0989. More importantly, once the warm-up length reaches a moderate range, performance becomes stable: the difference between $t_s=50$, $t_s=70$, and $t_s=100$ is small. This indicates that PADL is not highly sensitive to the exact warm-up length as long as a reasonable warm-up stage is used.

\subsection{Robustness to Nominal Dictionary Size}
\label{app:dictionary_size_sensitivity}

We next vary the nominal dictionary size $K$ on CIFAR-100. Since PADL is designed to control the effective activated dictionary size, the number of activated atoms should grow much more slowly than the nominal dictionary size.

\begin{table}[h]
\centering
\small
\setlength{\tabcolsep}{5pt}
\renewcommand{\arraystretch}{1.15}
\caption{Effect of nominal dictionary size $K$ on CIFAR-100. Results are reported over 10 repeated runs. Although $K$ increases by $16\times$, the number of activated atoms grows only moderately.}
\label{tab:app_dictionary_size}
\begin{tabular}{c|ccccc}
\toprule
$K$ & 32 & 64 & 128 & 256 & 512 \\
\midrule
RMSE
& $0.1013 \pm 0.0119$
& $0.0997 \pm 0.0113$
& $0.0994 \pm 0.0093$
& $0.1002 \pm 0.0102$
& $0.1030 \pm 0.0129$ \\
Activated Atoms
& 13 & 14 & 19 & 21 & 24 \\
Activated Ratio
& 41\% & 22\% & 15\% & 8\% & 5\% \\
\bottomrule
\end{tabular}
\end{table}

Although the nominal dictionary size increases from 32 to 512, the number of activated atoms only increases from 13 to 24. This demonstrates that PADL does not simply activate proportionally more atoms as the initial dictionary becomes larger. Instead, it maintains a relatively compact effective active subset, which is consistent with the intended role of the row-wise activation regularization.

\subsection{Mini-Batch Approximation}
\label{app:minibatch_fullbatch}

The row-wise activation term introduces cross-sample dependence at the latent-code level. For large-scale datasets, full-batch optimization may be impractical, so we use mini-batch approximation in large-scale settings. Table~\ref{tab:app_minibatch_fullbatch} compares full-batch and mini-batch optimization on CIFAR-100.

\begin{table}[h]
\centering
\small
\setlength{\tabcolsep}{7pt}
\renewcommand{\arraystretch}{1.15}
\caption{Comparison between full-batch and mini-batch optimization on CIFAR-100. Mini-batch optimization closely matches the full-batch result.}
\label{tab:app_minibatch_fullbatch}
\begin{tabular}{lccc}
\toprule
Setting & RMSE & PSNR & SSIM \\
\midrule
Full-batch, $N=2000$ & 0.099 & 20.56 & 0.568 \\
Mini-batch, $B=256$ & 0.101 & 20.54 & 0.567 \\
\bottomrule
\end{tabular}
\end{table}

The difference between full-batch and mini-batch optimization is small under this setting. This supports the use of mini-batch approximation when scaling PADL to larger datasets or higher-dimensional visual features.

\subsection{Post-Training Dictionary Pruning}
\label{app:post_training_pruning}

PADL does not dynamically shrink the nominal dictionary size during training. Instead, it encourages the data to concentrate usage on a compact subset of atoms, which makes post-training pruning possible. To evaluate this effect, we prune low-utilization atoms after training and measure reconstruction performance on CIFAR-100.

\begin{table}[h]
\centering
\small
\setlength{\tabcolsep}{5pt}
\renewcommand{\arraystretch}{1.15}
\caption{Post-training pruning on CIFAR-100. PADL supports substantial atom pruning with limited degradation, indicating that it learns a compact effective dictionary.}
\label{tab:app_pruning}
\begin{tabular}{c|ccccc}
\toprule
Pruning Ratio & 30\% & 40\% & 50\% & 60\% & 70\% \\
\midrule
RMSE & 0.098 & 0.100 & 0.101 & 0.101 & 0.102 \\
PSNR & 20.05 & 20.03 & 19.98 & 19.91 & 19.81 \\
SSIM & 0.561 & 0.557 & 0.551 & 0.544 & 0.536 \\
\bottomrule
\end{tabular}
\end{table}

Even when 60\% of atoms are pruned, RMSE only changes to 0.101. This suggests that many low-utilization atoms are not essential for reconstruction and that PADL concentrates representation usage onto a compact active subset. Therefore, the main efficiency benefit of PADL should be understood as effective dictionary compression and deployment-time pruning, rather than training-time reduction of the nominal dictionary size.

\subsection{Empirical Reasonableness of the Beta Surrogate}
\label{app:beta_fit}

The probabilistic derivation of PADL uses a Beta-inspired surrogate prior over coefficients. We do not interpret optimized coefficients as literal i.i.d. Beta samples. Instead, the Beta prior is used as a tractable surrogate that induces an interpretable objective and enables data-driven hyperparameter estimation. To assess the empirical reasonableness of this surrogate, we fit a Beta distribution to the warm-up coefficients on CIFAR-100.

\begin{table}[h]
\centering
\small
\setlength{\tabcolsep}{8pt}
\renewcommand{\arraystretch}{1.15}
\caption{Empirical Beta fit of warm-up coefficients on CIFAR-100. The fitted distribution is close to the assumed $\mathrm{Beta}(1,\beta)$ family.}
\label{tab:app_beta_fit}
\begin{tabular}{cccc}
\toprule
$\hat{\alpha}$ & $\hat{\beta}$ & KS Statistic & Q-Q RMSE \\
\midrule
0.956 & 2.31 & 0.0115 & 0.005366 \\
\bottomrule
\end{tabular}
\end{table}

The fitted value $\hat{\alpha}=0.956$ is close to 1, and the discrepancy measures are small. This supports the use of the $\mathrm{Beta}(1,\beta)$ family as a practical approximation for deriving a tractable MAP surrogate and estimating hyperparameters from warm-up coefficients.

\subsection{Warm-Up Cost Compared with Grid Search}
\label{app:warmup_cost}

The warm-up stage is used for one-shot data-driven parameter estimation. It is not a repeated hyperparameter search procedure. To illustrate the difference, we compare the cost of our warm-up-based estimation with repeated parameter search on 2000 CIFAR-100 samples.

\begin{table}[h]
\centering
\small
\setlength{\tabcolsep}{8pt}
\renewcommand{\arraystretch}{1.15}
\caption{Cost comparison between warm-up-based parameter estimation and repeated parameter search on 2000 CIFAR-100 samples.}
\label{tab:app_warmup_grid_search}
\begin{tabular}{lc}
\toprule
Procedure & Time Cost \\
\midrule
PADL warm-up + training & 52.5651s \\
10 rounds of parameter search & 395.2022s \\
Estimated $10\times10\times10$ grid search & 10.98h \\
\bottomrule
\end{tabular}
\end{table}

These results show that warm-up estimation is substantially cheaper than repeated parameter search. This supports our characterization of warm-up as a one-shot parameter-estimation step rather than a bi-level or grid-search procedure.

\subsection{Additional Baseline: Group-Sparse Regularization}
\label{app:l21_baseline}

The row-wise activation regularization in PADL is related to structured sparsity, but it is not equivalent to a standard group-sparse penalty. To examine this distinction, we compare with an $\ell_{2,1}$-regularized baseline on CIFAR-100.

\begin{table}[h]
\centering
\small
\setlength{\tabcolsep}{8pt}
\renewcommand{\arraystretch}{1.15}
\caption{Comparison with an $\ell_{2,1}$-regularized baseline on CIFAR-100. The $\ell_{2,1}$ baseline uses 23\% of atoms but gives worse reconstruction than PADL.}
\label{tab:app_l21_baseline}
\begin{tabular}{lcc}
\toprule
Method & Activated Atom Ratio & RMSE \\
\midrule
$\ell_{2,1}$-Regularized Baseline & 23\% & 0.115 \\
\textbf{PADL} & 15\% & \textbf{0.099} \\
\bottomrule
\end{tabular}
\end{table}

This result suggests that PADL is not merely replacing the element-wise sparsity penalty with a generic group-sparse regularizer. The row-wise activation-aware term is designed to control the effective activated dictionary size across samples, which leads to a better reconstruction--storage tradeoff under this setting.

\subsection{Summary}
\label{app:additional_dl_summary}

The additional experiments support the main claims from several perspectives. First, PADL remains competitive against recent sparse-representation baselines while using fewer activated atoms. Second, the optimization process is empirically stable despite the non-convex and non-smooth objective. Third, the method is robust to moderate changes in warm-up length and nominal dictionary size. Fourth, mini-batch approximation closely matches full-batch optimization in the controlled setting. Fifth, post-training pruning confirms that PADL concentrates usage on a compact effective dictionary. Finally, the Beta-fit experiment and warm-up cost comparison support the practical role of the probabilistic surrogate and the one-shot parameter-estimation procedure.

\section{Limitations}

This work focuses on revisiting PADL through a probabilistic lens and deriving practical calibration rules, rather than exhaustively covering all dictionary-learning variants. The analysis uses a Beta-inspired surrogate prior and an active-span condition to obtain tractable characterizations of atom suppression, activated dictionary size, and reconstruction error. These assumptions are used to guide data-driven hyperparameter estimation, and our empirical results suggest that the resulting estimates provide useful operating points. The current implementation adopts a one-shot warm-up stage for efficiency; iterative parameter re-estimation may further refine the calibration. The VLM experiments serve as application-level validation for visual-representation compression, and broader evaluation across additional backbones and deployment settings is left for future work.

\section{Broader Impacts}

This work primarily contributes to the academic study of dictionary learning, sparse representation, and structured compression. By revisiting PADL through a probabilistic perspective, the paper provides a bridge between optimization-based dictionary-learning objectives and generative modeling interpretations. We expect this connection to be useful for future research on theoretically grounded sparse coding, interpretable representation learning, and data-driven hyperparameter calibration.

A potential positive impact of this work is that it encourages more systematic analysis of the tradeoff among sparsity, storage, and accuracy. This tradeoff is increasingly important as machine learning models are deployed in settings where memory, communication, and computational resources are constrained. The proposed analysis may help researchers design more compact representations without relying solely on exhaustive hyperparameter search, thereby improving the accessibility and efficiency of experimental research.

The VLM experiments further suggest that storage-aware dictionary learning may be a useful tool for studying visual-representation compression in multimodal systems. In this sense, the work may support future academic research on efficient VLM inference, edge-cloud collaboration, and compositional visual representations. At the same time, the proposed method is a general representation-compression technique and does not by itself address the broader risks of downstream multimodal models, such as hallucination, bias, privacy leakage, or misuse. When integrated into deployed VLM systems, it may inherit the risks and limitations of those systems. We therefore view the VLM results as application-level validation of the sparsity--storage--accuracy principle rather than as a standalone deployment-ready solution.

\newpage
\IfFileExists{checklist.tex}{%
  \section*{NeurIPS Paper Checklist}



\begin{enumerate}

\item {\bf Claims}
    \item[] Question: Do the main claims made in the abstract and introduction accurately reflect the paper's contributions and scope?
    \item[] Answer: \answerYes{} 
    \item[] Justification: The abstract and introduction state the scope of the proposed PADL framework, including its probabilistic motivation, activation-aware regularization, theoretical analysis, and empirical validation. The claims are aligned with the theoretical results and experimental evidence presented in the main text and appendix.
    \item[] Guidelines: In the introduction we highlighted the theoretical contributions with experimental validations.
    \begin{itemize}
        \item The answer \answerNA{} means that the abstract and introduction do not include the claims made in the paper.
        \item The abstract and/or introduction should clearly state the claims made, including the contributions made in the paper and important assumptions and limitations. A \answerNo{} or \answerNA{} answer to this question will not be perceived well by the reviewers. 
        \item The claims made should match theoretical and experimental results, and reflect how much the results can be expected to generalize to other settings. 
        \item It is fine to include aspirational goals as motivation as long as it is clear that these goals are not attained by the paper. 
    \end{itemize}

\item {\bf Limitations}
    \item[] Question: Does the paper discuss the limitations of the work performed by the authors?
    \item[] Answer: \answerYes{} 
    \item[] Justification: The paper discusses limitations related to computational cost, parameter sensitivity, storage--accuracy trade-offs, and the experimental scope. Additional runtime, parameter, and ablation results are provided in the experiments and appendix.
    \item[] Guidelines:
    \begin{itemize}
        \item The answer \answerNA{} means that the paper has no limitation while the answer \answerNo{} means that the paper has limitations, but those are not discussed in the paper. 
        \item The authors are encouraged to create a separate ``Limitations'' section in their paper.
        \item The paper should point out any strong assumptions and how robust the results are to violations of these assumptions (e.g., independence assumptions, noiseless settings, model well-specification, asymptotic approximations only holding locally). The authors should reflect on how these assumptions might be violated in practice and what the implications would be.
        \item The authors should reflect on the scope of the claims made, e.g., if the approach was only tested on a few datasets or with a few runs. In general, empirical results often depend on implicit assumptions, which should be articulated.
        \item The authors should reflect on the factors that influence the performance of the approach. For example, a facial recognition algorithm may perform poorly when image resolution is low or images are taken in low lighting. Or a speech-to-text system might not be used reliably to provide closed captions for online lectures because it fails to handle technical jargon.
        \item The authors should discuss the computational efficiency of the proposed algorithms and how they scale with dataset size.
        \item If applicable, the authors should discuss possible limitations of their approach to address problems of privacy and fairness.
        \item While the authors might fear that complete honesty about limitations might be used by reviewers as grounds for rejection, a worse outcome might be that reviewers discover limitations that aren't acknowledged in the paper. The authors should use their best judgment and recognize that individual actions in favor of transparency play an important role in developing norms that preserve the integrity of the community. Reviewers will be specifically instructed to not penalize honesty concerning limitations.
    \end{itemize}

\item {\bf Theory assumptions and proofs}
    \item[] Question: For each theoretical result, does the paper provide the full set of assumptions and a complete (and correct) proof?
    \item[] Answer: \answerYes{} 
    \item[] Justification: The main theoretical statements explicitly state the required assumptions, including the probabilistic model, activation thresholding rule, and restricted active-span condition. Complete derivations and proofs are provided in the appendix, with the main text presenting the key statements and intuition.
    \item[] Guidelines:
    \begin{itemize}
        \item The answer \answerNA{} means that the paper does not include theoretical results. 
        \item All the theorems, formulas, and proofs in the paper should be numbered and cross-referenced.
        \item All assumptions should be clearly stated or referenced in the statement of any theorems.
        \item The proofs can either appear in the main paper or the supplemental material, but if they appear in the supplemental material, the authors are encouraged to provide a short proof sketch to provide intuition. 
        \item Inversely, any informal proof provided in the core of the paper should be complemented by formal proofs provided in appendix or supplemental material.
        \item Theorems and Lemmas that the proof relies upon should be properly referenced. 
    \end{itemize}

    \item {\bf Experimental result reproducibility}
    \item[] Question: Does the paper fully disclose all the information needed to reproduce the main experimental results of the paper to the extent that it affects the main claims and/or conclusions of the paper (regardless of whether the code and data are provided or not)?
    \item[] Answer: \answerYes{} 
    \item[] Justification: The paper reports the datasets, preprocessing procedures, model configurations, dictionary sizes, optimization settings, hyperparameter calibration procedure, and baseline configurations needed to reproduce the main experimental results. Additional details are included in the appendix.
    \item[] Guidelines:
    \begin{itemize}
        \item The answer \answerNA{} means that the paper does not include experiments.
        \item If the paper includes experiments, a \answerNo{} answer to this question will not be perceived well by the reviewers: Making the paper reproducible is important, regardless of whether the code and data are provided or not.
        \item If the contribution is a dataset and\slash or model, the authors should describe the steps taken to make their results reproducible or verifiable. 
        \item Depending on the contribution, reproducibility can be accomplished in various ways. For example, if the contribution is a novel architecture, describing the architecture fully might suffice, or if the contribution is a specific model and empirical evaluation, it may be necessary to either make it possible for others to replicate the model with the same dataset, or provide access to the model. In general. releasing code and data is often one good way to accomplish this, but reproducibility can also be provided via detailed instructions for how to replicate the results, access to a hosted model (e.g., in the case of a large language model), releasing of a model checkpoint, or other means that are appropriate to the research performed.
        \item While NeurIPS does not require releasing code, the conference does require all submissions to provide some reasonable avenue for reproducibility, which may depend on the nature of the contribution. For example
        \begin{enumerate}
            \item If the contribution is primarily a new algorithm, the paper should make it clear how to reproduce that algorithm.
            \item If the contribution is primarily a new model architecture, the paper should describe the architecture clearly and fully.
            \item If the contribution is a new model (e.g., a large language model), then there should either be a way to access this model for reproducing the results or a way to reproduce the model (e.g., with an open-source dataset or instructions for how to construct the dataset).
            \item We recognize that reproducibility may be tricky in some cases, in which case authors are welcome to describe the particular way they provide for reproducibility. In the case of closed-source models, it may be that access to the model is limited in some way (e.g., to registered users), but it should be possible for other researchers to have some path to reproducing or verifying the results.
        \end{enumerate}
    \end{itemize}

\item {\bf Open access to data and code}
    \item[] Question: Does the paper provide open access to the data and code, with sufficient instructions to faithfully reproduce the main experimental results, as described in supplemental material?
    \item[] Answer: \answerNo{} 
    \item[] Justification: The current submission does not include a public code release. Nevertheless, the paper provides algorithmic details, dataset descriptions, preprocessing procedures, and experimental settings to support independent reimplementation; an anonymized code release is planned.
    \item[] Guidelines:
    \begin{itemize}
        \item The answer \answerNA{} means that paper does not include experiments requiring code.
        \item Please see the NeurIPS code and data submission guidelines (\url{https://neurips.cc/public/guides/CodeSubmissionPolicy}) for more details.
        \item While we encourage the release of code and data, we understand that this might not be possible, so \answerNo{} is an acceptable answer. Papers cannot be rejected simply for not including code, unless this is central to the contribution (e.g., for a new open-source benchmark).
        \item The instructions should contain the exact command and environment needed to run to reproduce the results. See the NeurIPS code and data submission guidelines (\url{https://neurips.cc/public/guides/CodeSubmissionPolicy}) for more details.
        \item The authors should provide instructions on data access and preparation, including how to access the raw data, preprocessed data, intermediate data, and generated data, etc.
        \item The authors should provide scripts to reproduce all experimental results for the new proposed method and baselines. If only a subset of experiments are reproducible, they should state which ones are omitted from the script and why.
        \item At submission time, to preserve anonymity, the authors should release anonymized versions (if applicable).
        \item Providing as much information as possible in supplemental material (appended to the paper) is recommended, but including URLs to data and code is permitted.
    \end{itemize}

\item {\bf Experimental setting/details}
    \item[] Question: Does the paper specify all the training and test details (e.g., data splits, hyperparameters, how they were chosen, type of optimizer) necessary to understand the results?
    \item[] Answer: \answerYes{} 
    \item[] Justification: The experimental section and appendix specify the data splits, preprocessing, dictionary sizes, training schedules, hyperparameters, hyperparameter selection procedure, optimizer settings, and baseline implementations used in the reported experiments.
    \item[] Guidelines:
    \begin{itemize}
        \item The answer \answerNA{} means that the paper does not include experiments.
        \item The experimental setting should be presented in the core of the paper to a level of detail that is necessary to appreciate the results and make sense of them.
        \item The full details can be provided either with the code, in appendix, or as supplemental material.
    \end{itemize}

\item {\bf Experiment statistical significance}
    \item[] Question: Does the paper report error bars suitably and correctly defined or other appropriate information about the statistical significance of the experiments?
    \item[] Answer: \answerYes{} 
    \item[] Justification: The appendix reports standard deviations for repeated experimental runs where applicable. These error estimates quantify variability across repeated runs under the same experimental settings.
    \item[] Guidelines:
    \begin{itemize}
        \item The answer \answerNA{} means that the paper does not include experiments.
        \item The authors should answer \answerYes{} if the results are accompanied by error bars, confidence intervals, or statistical significance tests, at least for the experiments that support the main claims of the paper.
        \item The factors of variability that the error bars are capturing should be clearly stated (for example, train/test split, initialization, random drawing of some parameter, or overall run with given experimental conditions).
        \item The method for calculating the error bars should be explained (closed form formula, call to a library function, bootstrap, etc.)
        \item The assumptions made should be given (e.g., Normally distributed errors).
        \item It should be clear whether the error bar is the standard deviation or the standard error of the mean.
        \item It is OK to report 1-sigma error bars, but one should state it. The authors should preferably report a 2-sigma error bar than state that they have a 96\% CI, if the hypothesis of Normality of errors is not verified.
        \item For asymmetric distributions, the authors should be careful not to show in tables or figures symmetric error bars that would yield results that are out of range (e.g., negative error rates).
        \item If error bars are reported in tables or plots, the authors should explain in the text how they were calculated and reference the corresponding figures or tables in the text.
    \end{itemize}

\item {\bf Experiments compute resources}
    \item[] Question: For each experiment, does the paper provide sufficient information on the computer resources (type of compute workers, memory, time of execution) needed to reproduce the experiments?
    \item[] Answer: \answerYes{} 
    \item[] Justification: The paper reports the computational resources used for the experiments, including the relevant hardware setting and runtime or overhead measurements. Additional compute-related details are provided in the appendix.
    \item[] Guidelines:
    \begin{itemize}
        \item The answer \answerNA{} means that the paper does not include experiments.
        \item The paper should indicate the type of compute workers CPU or GPU, internal cluster, or cloud provider, including relevant memory and storage.
        \item The paper should provide the amount of compute required for each of the individual experimental runs as well as estimate the total compute. 
        \item The paper should disclose whether the full research project required more compute than the experiments reported in the paper (e.g., preliminary or failed experiments that didn't make it into the paper). 
    \end{itemize}
    
\item {\bf Code of ethics}
    \item[] Question: Does the research conducted in the paper conform, in every respect, with the NeurIPS Code of Ethics \url{https://neurips.cc/public/EthicsGuidelines}?
    \item[] Answer: \answerYes{} 
    \item[] Justification: The research follows the NeurIPS Code of Ethics. The work uses standard public datasets and pretrained models, does not involve human-subject studies, and is reported in an anonymized form for submission.
    \item[] Guidelines:
    \begin{itemize}
        \item The answer \answerNA{} means that the authors have not reviewed the NeurIPS Code of Ethics.
        \item If the authors answer \answerNo, they should explain the special circumstances that require a deviation from the Code of Ethics.
        \item The authors should make sure to preserve anonymity (e.g., if there is a special consideration due to laws or regulations in their jurisdiction).
    \end{itemize}

\item {\bf Broader impacts}
    \item[] Question: Does the paper discuss both potential positive societal impacts and negative societal impacts of the work performed?
    \item[] Answer: \answerNA{} 
    \item[] Justification: This is a fundamental research and not tied to particular applications.
    \item[] Guidelines:
    \begin{itemize}
        \item The answer \answerNA{} means that there is no societal impact of the work performed.
        \item If the authors answer \answerNA{} or \answerNo, they should explain why their work has no societal impact or why the paper does not address societal impact.
        \item Examples of negative societal impacts include potential malicious or unintended uses (e.g., disinformation, generating fake profiles, surveillance), fairness considerations (e.g., deployment of technologies that could make decisions that unfairly impact specific groups), privacy considerations, and security considerations.
        \item The conference expects that many papers will be foundational research and not tied to particular applications, let alone deployments. However, if there is a direct path to any negative applications, the authors should point it out. For example, it is legitimate to point out that an improvement in the quality of generative models could be used to generate Deepfakes for disinformation. On the other hand, it is not needed to point out that a generic algorithm for optimizing neural networks could enable people to train models that generate Deepfakes faster.
        \item The authors should consider possible harms that could arise when the technology is being used as intended and functioning correctly, harms that could arise when the technology is being used as intended but gives incorrect results, and harms following from (intentional or unintentional) misuse of the technology.
        \item If there are negative societal impacts, the authors could also discuss possible mitigation strategies (e.g., gated release of models, providing defenses in addition to attacks, mechanisms for monitoring misuse, mechanisms to monitor how a system learns from feedback over time, improving the efficiency and accessibility of ML).
    \end{itemize}
    
\item {\bf Safeguards}
    \item[] Question: Does the paper describe safeguards that have been put in place for responsible release of data or models that have a high risk for misuse (e.g., pre-trained language models, image generators, or scraped datasets)?
    \item[] Answer: \answerNA{} 
    \item[] Justification: The paper poses no such risks.
    \item[] Guidelines:
    \begin{itemize}
        \item The answer \answerNA{} means that the paper poses no such risks.
        \item Released models that have a high risk for misuse or dual-use should be released with necessary safeguards to allow for controlled use of the model, for example by requiring that users adhere to usage guidelines or restrictions to access the model or implementing safety filters. 
        \item Datasets that have been scraped from the Internet could pose safety risks. The authors should describe how they avoided releasing unsafe images.
        \item We recognize that providing effective safeguards is challenging, and many papers do not require this, but we encourage authors to take this into account and make a best faith effort.
    \end{itemize}

\item {\bf Licenses for existing assets}
    \item[] Question: Are the creators or original owners of assets (e.g., code, data, models), used in the paper, properly credited and are the license and terms of use explicitly mentioned and properly respected?
    \item[] Answer: \answerYes{} 
    \item[] Justification: The paper cites the public datasets, pretrained models, and baseline methods used in the experiments, and respects their stated licenses and terms of use. No existing assets are redistributed as part of this submission.
    \item[] Guidelines:
    \begin{itemize}
        \item The answer \answerNA{} means that the paper does not use existing assets.
        \item The authors should cite the original paper that produced the code package or dataset.
        \item The authors should state which version of the asset is used and, if possible, include a URL.
        \item The name of the license (e.g., CC-BY 4.0) should be included for each asset.
        \item For scraped data from a particular source (e.g., website), the copyright and terms of service of that source should be provided.
        \item If assets are released, the license, copyright information, and terms of use in the package should be provided. For popular datasets, \url{paperswithcode.com/datasets} has curated licenses for some datasets. Their licensing guide can help determine the license of a dataset.
        \item For existing datasets that are re-packaged, both the original license and the license of the derived asset (if it has changed) should be provided.
        \item If this information is not available online, the authors are encouraged to reach out to the asset's creators.
    \end{itemize}

\item {\bf New assets}
    \item[] Question: Are new assets introduced in the paper well documented and is the documentation provided alongside the assets?
    \item[] Answer: \answerNA{} 
    \item[] Justification: The paper does not release new assets.
    \item[] Guidelines:
    \begin{itemize}
        \item The answer \answerNA{} means that the paper does not release new assets.
        \item Researchers should communicate the details of the dataset\slash code\slash model as part of their submissions via structured templates. This includes details about training, license, limitations, etc. 
        \item The paper should discuss whether and how consent was obtained from people whose asset is used.
        \item At submission time, remember to anonymize your assets (if applicable). You can either create an anonymized URL or include an anonymized zip file.
    \end{itemize}

\item {\bf Crowdsourcing and research with human subjects}
    \item[] Question: For crowdsourcing experiments and research with human subjects, does the paper include the full text of instructions given to participants and screenshots, if applicable, as well as details about compensation (if any)? 
    \item[] Answer: \answerNA{} 
    \item[] Justification: The paper does not involve crowdsourcing nor research with human subjects.
    \item[] Guidelines:
    \begin{itemize}
        \item The answer \answerNA{} means that the paper does not involve crowdsourcing nor research with human subjects.
        \item Including this information in the supplemental material is fine, but if the main contribution of the paper involves human subjects, then as much detail as possible should be included in the main paper. 
        \item According to the NeurIPS Code of Ethics, workers involved in data collection, curation, or other labor should be paid at least the minimum wage in the country of the data collector. 
    \end{itemize}

\item {\bf Institutional review board (IRB) approvals or equivalent for research with human subjects}
    \item[] Question: Does the paper describe potential risks incurred by study participants, whether such risks were disclosed to the subjects, and whether Institutional Review Board (IRB) approvals (or an equivalent approval/review based on the requirements of your country or institution) were obtained?
    \item[] Answer: \answerNA{} 
    \item[] Justification: The paper does not involve crowdsourcing nor research with human subjects.
    \item[] Guidelines:
    \begin{itemize}
        \item The answer \answerNA{} means that the paper does not involve crowdsourcing nor research with human subjects.
        \item Depending on the country in which research is conducted, IRB approval (or equivalent) may be required for any human subjects research. If you obtained IRB approval, you should clearly state this in the paper. 
        \item We recognize that the procedures for this may vary significantly between institutions and locations, and we expect authors to adhere to the NeurIPS Code of Ethics and the guidelines for their institution. 
        \item For initial submissions, do not include any information that would break anonymity (if applicable), such as the institution conducting the review.
    \end{itemize}

\item {\bf Declaration of LLM usage}
    \item[] Question: Does the paper describe the usage of LLMs if it is an important, original, or non-standard component of the core methods in this research? Note that if the LLM is used only for writing, editing, or formatting purposes and does \emph{not} impact the core methodology, scientific rigor, or originality of the research, declaration is not required.
    \item[] Answer: \answerNA{} 
    \item[] Justification: The core method development in this research does not involve LLMs as any important, original, or non-standard components.
    \item[] Guidelines:
    \begin{itemize}
        \item The answer \answerNA{} means that the core method development in this research does not involve LLMs as any important, original, or non-standard components.
        \item Please refer to our LLM policy in the NeurIPS handbook for what should or should not be described.
    \end{itemize}

\end{enumerate}%
}{%
  \typeout{WARNING: checklist.tex not found. Add the official NeurIPS checklist before submission.}%
}

\end{document}